\documentclass{article}
\usepackage{iclr2025_conference}

\usepackage[T1]{fontenc}
\usepackage[utf8]{inputenc}
\usepackage{microtype}
\usepackage{latexsym}
\usepackage{graphicx}
\usepackage{marvosym}
\usepackage{url}
\usepackage{booktabs}
\usepackage{amsmath,amssymb,amsfonts,amsthm}
\usepackage{nicefrac}
\usepackage{xcolor}
\usepackage[ruled,vlined]{algorithm2e}
\SetKwInput{KwInput}{Input}
\SetKwInput{KwOutput}{Output}
\usepackage{subcaption}
\usepackage{wrapfig}
\usepackage{bbm}
\usepackage{bm}
\usepackage{multirow}
\usepackage[table]{xcolor}
\usepackage{enumitem}
\usepackage{tcolorbox}
\tcbuselibrary{breakable}

\definecolor{uclablue}{rgb}{0.15, 0.45, 0.68}
\hypersetup{
  breaklinks,
  citecolor=uclablue,
  colorlinks=true,
}

\tcbset{
  example/.style={
    breakable,
    colback=white,
    colframe=black!55,
    colbacktitle=black!6,
    coltitle=black,
    fonttitle=\bfseries,
    boxrule=0.5pt,
    arc=1pt,
    left=5pt,
    right=5pt,
    top=5pt,
    bottom=5pt,
    before skip=6pt,
    after skip=6pt
  }
}

\newtheorem{proposition}{Proposition}

\title{TRACE: A Unified Rollout Budget Allocation Framework for Efficient Agentic Reinforcement Learning}

\author{
  \textbf{Heming Zou\textsuperscript{1,2}\thanks{Work completed during an internship at Tencent.}},
  \textbf{Qi Wang\textsuperscript{1}\thanks{Corresponding authors.}},
  \textbf{Yun Qu\textsuperscript{1,2}\footnotemark[1]},
  \textbf{Yuhang Jiang\textsuperscript{1}},
  \textbf{Lizhou Cai\textsuperscript{1}},
  \textbf{Yixiu Mao\textsuperscript{1}},\\
  \textbf{Ru Peng\textsuperscript{2}\footnotemark[1]},
  \textbf{Xin Xu\textsuperscript{2}},
  \textbf{Weijie Liu\textsuperscript{2}},
  \textbf{Kai Yang\textsuperscript{2}},
  \textbf{Saiyong Yang\textsuperscript{2}\footnotemark[2]},
  \textbf{Xiangyang Ji\textsuperscript{1}\footnotemark[2]}\\[2ex]
  \textbf{$^1$Tsinghua University} \quad
  \textbf{$^2$LLM Department, Tencent}\\[2ex]
  \centerline{%
  \begin{tabular}{c}
    \Letter\ \href{mailto:cheemswang@mail.tsinghua.edu.cn}{cheemswang@mail.tsinghua.edu.cn} \\[0.15ex]
    \href{mailto:stevesyang@tencent.com}{stevesyang@tencent.com} \\[0.15ex]
    \href{mailto:xyji@tsinghua.edu.cn}{xyji@tsinghua.edu.cn} \\
  \end{tabular}%
  }
}

\begin{document}
\maketitle

\begin{abstract}
  Reinforcement learning with verifiable rewards (RLVR) is a promising approach for enhancing reasoning and agentic behavior in large language models. However, rollout-intensive policy optimization is often limited by insufficient reward contrast, arising when overly simple or complex prompts generate low-variance feedback and when outcome-only rewards assign the same terminal assessment to every decision in a multi-turn rollout. Past efforts have focused on allocating available rollout resources to promising prompts, yet they only leverage sample informativeness at the prompt level and neglect variation in prefix-level informativeness across turns within the same rollout. This work targets multi-turn agentic RL by modeling each ReAct-style thought-action-observation turn as a semantically distinct node, allowing budget allocation to extend from prompt roots to turn-level prefixes with further continuations, which naturally forms tree-structured rollouts. We introduce Tree Rollout Allocation for Contrastive Exploration (TRACE), a unified rollout allocation framework that enhances reward contrast within a fixed sampling budget. Technically, TRACE allocates rollout budget to both prompt roots and intermediate prefixes that are most likely to yield mixed terminal rewards. A shared generalizable predictor estimates conditional success probability at these anchors from prefix histories to guide this allocation. The resulting adaptive tree structure enriches outcome-only feedback and amplifies the policy-update signal. Empirically, TRACE achieves competitive performance and efficiency gains on typical agentic benchmarks, e.g., improving Qwen3-14B Multi-Hop QA average accuracy by 2.8 points over competitive baselines at equal sampling cost.
\end{abstract}

\section{Introduction}
\label{sec:intro}

Reinforcement learning with verifiable rewards (RLVR) has become a key mechanism for incentivizing reasoning and agentic behavior in large language models (LLMs), using verifiable feedback to shape exploration over solution paths and environment interactions~\citep{jaech2024openai,guo2025deepseek}. Despite its success, RLVR is computationally expensive because rollouts require long Chain-of-Thought (CoT)~\citep{wei2022chain} generation for reasoning and interleaved environment interactions for agentic tasks. Deciding where to spend limited rollout budget is therefore a central design problem~\citep{zheng2025act,lin2025cppo}.

In practice, most scalable RLVR pipelines rely on outcome rewards rather than process rewards, avoiding hand-crafted intermediate reward rules that are vulnerable to reward hacking~\citep{shao2024deepseekmath}. The learning signal provided by such rollouts, however, is highly uneven in informativeness and sparse for credit assignment. At the prompt level, overly easy or hard instances tend to produce low-variance outcome rewards~\citep{yu2025dapo}. At the trajectory level, assigning a single terminal reward to a long rollout leaves little local contrast for credit assignment across generated tokens and agent turns~\citep{lightman2024let}. Prior work alleviates prompt-level inefficiency through prompt selection~\citep{zheng2025act,gao2025prompt,bae2026online}, which decides which prompts to sample, and rollout allocation~\citep{li2025knapsack}, which decides how many rollouts each prompt receives. Together, they can be unified as scoring a candidate set and assigning rollout budget toward prompts likely to produce contrastive outcome variation: a prompt assigned zero rollouts is skipped, whereas a selected prompt assigned a larger positive budget receives more independent root rollouts. Yet this allocation strategy acts only at the prompt root. Once a prompt receives budget, each rollout is still generated as an atomic trajectory.

Multi-turn agentic RL exposes a finer allocation unit, the turn. While single-turn reasoning can be segmented at token or sentence boundaries, ReAct-style~\citep{yao2022react} interaction naturally packages each decision as a thought-action-observation node. A complete rollout therefore exposes semantically meaningful prefixes that can be revisited for branching. These visited prefixes become candidate branching points for extra exploration. Some prefixes are unlikely to produce new outcome variation, whereas others are likely to yield counterfactual continuations with different final rewards. Technically, allocating budget at this level is not a matter of sampling more independent trajectories; it is a choice of which prefix occurrences should receive additional continuation branches.

\begin{wrapfigure}{r}{0.5\textwidth}
  \centering
  \vspace{-10pt}
  \includegraphics[width=\linewidth]{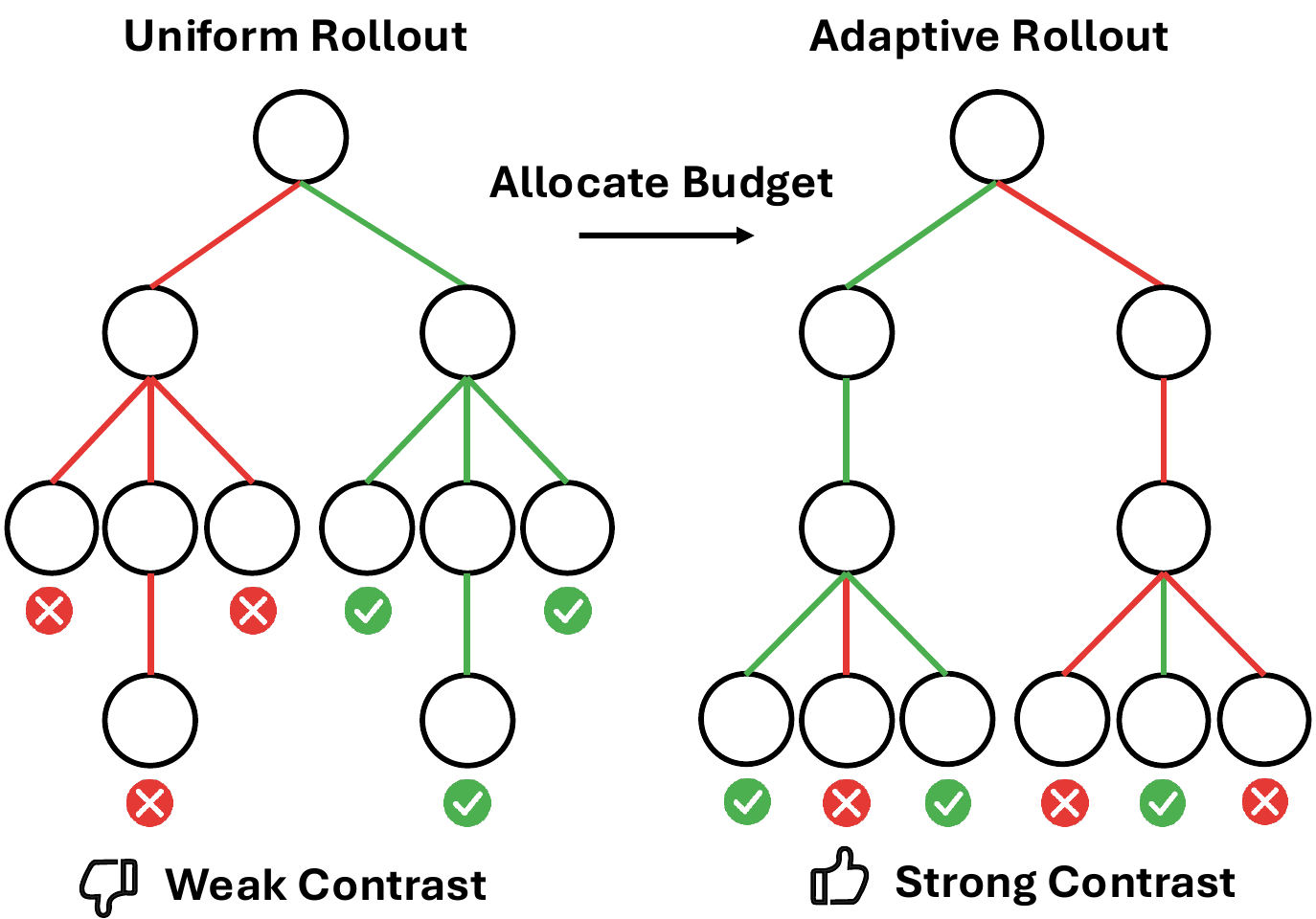}
  \caption{TRACE redirects a fixed rollout budget toward contrast-rich roots and prefixes, converting scalable outcome rewards into denser mixed-reward contrast and implicit stepwise preference pairs than uniform allocation.}
  \label{fig:trace_intuition}
  \vspace{-5pt}
\end{wrapfigure}

The above viewpoint naturally turns flat rollout collection into tree-structured rollouts. Prompt roots are depth-zero anchors and internal prefixes are non-root turn-level anchors, so root-level allocation becomes a special case of allocating budget over tree nodes. The shared principle is \emph{mixed-reward contrast construction}: allocate rollout budget to anchors whose descendant sets are most likely to contain both successful and failed outcomes. At the root, this principle unifies prompt filtering and rollout-count allocation, with zero budget rejecting a prompt and positive group budget setting its root rollout count. Inside an active prompt, the same principle allocates extra branches to prefixes likely to produce opposite-outcome siblings.

Figure~\ref{fig:trace_intuition} summarizes this view: under a fixed rollout budget, allocating branches to contrast-rich roots and prefixes yields more mixed terminal outcomes than uniform allocation over the rollout tree. We instantiate this idea as \textbf{T}ree \textbf{R}ollout \textbf{A}llocation for \textbf{C}ontrastive \textbf{E}xploration (\textbf{TRACE}), a unified rollout budget allocation framework for efficient RLVR in multi-turn agentic tasks. TRACE first performs global root allocation over a large candidate pool, then performs local tree expansion over the prefix occurrences exposed by the resulting rollouts. Both stages are guided by a shared generalizable predictor that estimates conditional success probability from prefix histories. The resulting rollout trees turn scalable outcome rewards into implicit stepwise preference pairs for tree-aware policy optimization. Across three agentic settings: Mathematical Reasoning, Multi-Hop QA, and Function Calling, TRACE improves accuracy over uniform sampling under the same rollout budget by directing samples toward anchors with useful reward contrast.

In summary, this work's contribution is three-fold:
\begin{enumerate}[topsep=0pt,itemsep=0ex,leftmargin=3ex]
    \item We formulate \textbf{mixed-reward contrast allocation} as a unified view of sample-efficient RLVR, incorporating prompt filtering, rollout-count allocation, and prefix-level branching as strategic budget decisions over rollout-tree anchors.
    \item We derive \textbf{root and prefix allocation utilities} based on a shared descendant-set objective, distributing budget according to the likelihood that an anchor's descendants will yield both successful and failed outcomes.
    \item We instantiate this view in \textbf{TRACE}, a predictor-guided, budget-constrained tree rollout framework that integrates global root allocation with local prefix expansion, enhancing implicit credit signals for sample-efficient policy optimization.
\end{enumerate}

\section{Preliminaries}
\label{sec:prelim}

\subsection{Multi-Turn RLVR as Prefix Histories}
\label{sec:react}

This work investigates multi-turn agentic reinforcement learning within the ReAct framework~\citep{yao2022react}. 
We define the input prompt as the root node $x$ of the rollout tree. 
At each turn $t=1,\ldots,T$, the model produces a thought $\tau_t$ and an action $a_t$, to which the environment responds with an output $o_t$.
We then encapsulate the complete turn as $n_t := \langle \tau_t, a_t, o_t \rangle$. 
The history available after turn $t$ is described as
\begin{equation}
  \mathcal{H}_t := (x, n_1, \ldots, n_t),
  \qquad
  \mathcal{H}_0 := x.
  \label{eq:history}
\end{equation}
With policy model $\pi_\theta$ and environment transition $P_{\mathrm{env}}$, the rollout dynamics can be expressed as the following hierarchical generative process:
\begin{equation}
  (\tau_t, a_t) \sim \pi_\theta(\cdot \mid \mathcal{H}_{t-1}),
  \qquad
  o_t \sim P_{\mathrm{env}}(\cdot \mid \mathcal{H}_{t-1}, \tau_t, a_t).
  \label{eq:react_rollout}
\end{equation}
A complete rollout is identified by its terminal history $\mathcal{H}_T$, and the RLVR framework assigns a binary terminal reward $r(\mathcal{H}_T)\in\{0,1\}$. 
Thus, the underlying training objective is formulated as
\begin{equation}
  \max_{\theta}
  \; \mathbb{E}_{x \sim \mathcal{D},\, \mathcal{H}_T \sim \mathbb{P}_{\pi_\theta}(\cdot \mid x)}\big[r(\mathcal{H}_T)\big].
  \label{eq:rlvr_objective}
\end{equation}
For any prefix $\mathcal{H}_t$, we define the conditional success probability of solving the prompt as
\begin{equation}
  V_t^\pi := \mathbb{E}_{\pi}\big[r(\mathcal{H}_T) \mid \mathcal{H}_t\big].
  \label{eq:prefix_value}
\end{equation}
A prompt is the shortest history $\mathcal{H}_0$, while later prefixes add the thoughts, actions, and observations already produced during the rollout.

\subsection{Initialize-then-Expand Rollout Trees}
\label{sec:trace_tree_construction}

To remain suited for parallelized LLM inference engines, tree rollout methods commonly follow an initialize-then-expand construction~\citep{ji2025tree,hou2025treerl}. After an allocation rule assigns each prompt $x_i$ a root rollout count $m_i$, Stage~1 samples complete bare rollouts. For rollout $j$, let $T_{i,j}$ be its terminal turn index and let $\mathcal{H}_{i,j,t}$ denote its prefix after turn $t$:
\begin{equation}
  \mathcal{H}_{i,j,T_{i,j}}\sim \pi_\theta(\cdot\mid x_i),
  \; j=1,\ldots,m_i,
  \label{eq:bare_rollout_generation}
\end{equation}
and records terminal rewards $r_{i,j}$.
Setting $m_i=0$ skips prompt $x_i$ in Stage~1, so Eq.~\eqref{eq:bare_rollout_generation} is not executed and no bare rollout index $j$ is defined.
Each visited prefix $\mathcal{H}_{i,j,t}$ can then serve as a branching anchor: if prefix expansion assigns $K_{i,j,t}$ continuations to it, Stage~2 samples terminal suffixes
\begin{equation}
  S_{i,j,t}^{(b)}
  \sim
  \pi_\theta(\cdot\mid \mathcal{H}_{i,j,t}),
  \; b=1,\ldots,K_{i,j,t}.
  \label{eq:prefix_branch_generation}
\end{equation}
The original suffix is the factual branch under the prefix, while newly sampled suffixes $S_{i,j,t}^{(b)}$ are counterfactual branches. This abstraction makes prompt roots and internal prefixes comparable budget-allocation anchors: both are conditioning contexts with terminal descendants below them.
For an active prompt $x_i$, the corresponding nonterminal prefix-anchor index set is
\begin{equation}
  \mathcal{A}_i:=\{(j,t):1\le j\le m_i,\;1\le t<T_{i,j}\},
  \label{eq:prefix_anchor_set}
\end{equation}
which excludes terminal leaves and records prefix occurrences available for continuation branching.
For a visited prefix occurrence $(j,t)\in\mathcal{A}_i$, setting $K_{i,j,t}=0$ skips Stage~2 branching at that anchor, so Eq.~\eqref{eq:prefix_branch_generation} is not executed and no continuation index $b$ is defined.

\section{Rollout Allocation through the Lens of Contrast Construction}
\label{sec:contrast_allocation}

In outcome-reward RLVR~\citep{guo2025deepseek}, more rollouts are not automatically more informative. A broad line of process-supervision, preference-learning, and tree-based RL methods reflects the same lesson: learning becomes denser when the model can compare alternatives rather than consume isolated terminal labels~\citep{lightman2024let,lai2024step,zhang2024rest}. If all descendants of an anchor receive the same terminal reward, group-relative or tree-aware updates have little local preference signal~\citep{shao2024deepseekmath,yang2025treerpo}: sampled continuations do not tell the optimizer which decisions should be preferred under the same conditioning context. When rollouts sharing the same prompt root or prefix disagree in terminal outcome while only the continuation differs, they yield implicit pairwise preferences over sibling continuations rather than a single sparse terminal reward shared across the entire rollout. We therefore treat mixed-reward contrast construction as the allocation objective: a fixed rollout budget is useful to the extent that it induces such pairwise comparison at prompt roots and visited prefixes.

\begin{figure*}[t]
  \centering
  \begin{subfigure}[t]{0.32\textwidth}
    \centering
    \includegraphics[width=\linewidth]{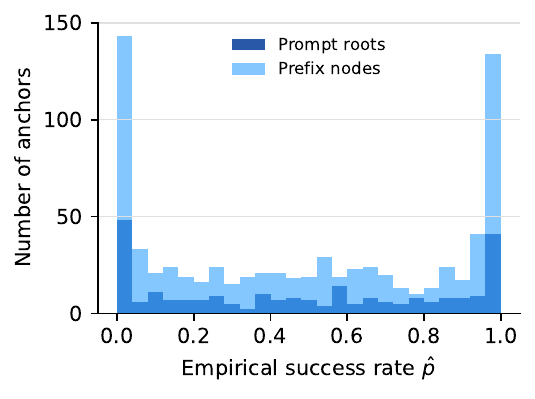}
    \caption{Success-rate heterogeneity}
    \label{fig:contrast_heterogeneity}
  \end{subfigure}
  \hfill
  \begin{subfigure}[t]{0.32\textwidth}
    \centering
    \includegraphics[width=\linewidth]{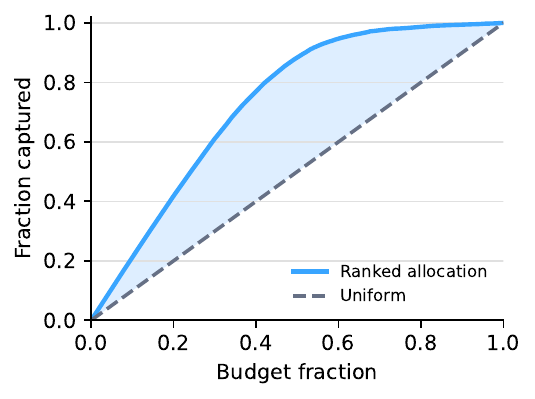}
    \caption{Contrast captured by budget}
    \label{fig:allocation_efficiency}
  \end{subfigure}
  \hfill
  \begin{subfigure}[t]{0.32\textwidth}
    \centering
    \includegraphics[width=\linewidth]{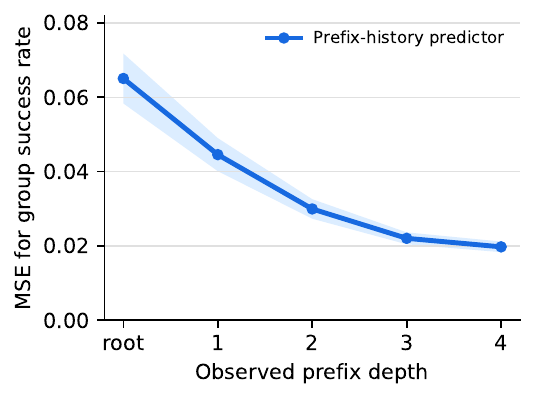}
    \caption{Prefix-conditioned prediction}
    \label{fig:prefix_prediction}
  \end{subfigure}
  \caption{\textbf{Contrastive-allocation diagnostics.} Using rollout results collected over several steps with Qwen3-8B under the Multi-Hop QA (HotpotQA) tree-sampling setting, panel (a) shows that many prompt-root and prefix anchors have empirical success rate $\hat p$ near $0$ or $1$, where outcome contrast is scarce. Panel (b) measures each anchor's pair contrast as $\hat p_h(1-\hat p_h)$; the $x$-axis is the fraction of rollout budget assigned to the top-ranked anchors, and the $y$-axis is the fraction of total pair contrast $\sum_h \hat p_h(1-\hat p_h)$ covered by those anchors, compared with uniform allocation. Panel (c) evaluates the same prefix-conditioned predictor at different observed prefix depths, showing lower group-success MSE at deeper depths and more reliable prefix-conditioned allocation.}
  \label{fig:contrast_allocation}
  \vspace{-5pt}
\end{figure*}

\subsection{Prefix Information for Better Difficulty Prediction}
\label{sec:prefix_prediction}

Prior sample-efficient RLVR methods use prompt-level difficulty prediction to filter prompts and allocate root rollouts~\citep{zheng2025act,gao2025prompt,bae2026online,li2025knapsack}. This motivates the pre-sampling prediction problem at prefix anchors in our tree setting: just as prompt roots are scored before Stage~1 rollout allocation, visited prefixes should be scored before Stage~2 continuation allocation. Each prefix history $\mathcal{H}_t$ refines the prompt root with observed interaction, so it conditions the same conditional success probability $V_t^\pi$ on richer state information and supports sharper forecasts of group outcome variation and downstream reward contrast.

\begin{proposition}[Prefix information improves group difficulty prediction]
\label{prop:prefix_predictability}
For $m\ge1$, let $\bar R_{t,m}$ be the average terminal reward of $m$ independent continuations sampled from prefix $\mathcal{H}_t$. For any squared-loss predictor $f$ that maps a history to a scalar estimate, define $\mathcal{E}_{t,m}^\star:=\inf_f \mathbb{E}[(\bar R_{t,m}-f(\mathcal{H}_t))^2]$. Then $\mathbb{E}[\bar R_{t,m}\mid\mathcal{H}_t]=V_t^\pi$, and for any $t<T$,
\begin{equation}
  \mathcal{E}_{t+1,m}^\star
  \le
  \mathcal{E}_{t,m}^\star .
\end{equation}
Consequently, $\mathcal{E}_{t,m}^\star\le\mathcal{E}_{0,m}^\star$ for all $t$.
The cases $t=0$ and $m=1$ recover prompt-level difficulty prediction and single-rollout reward prediction, respectively.
\end{proposition}

Intuitively, Proposition~\ref{prop:prefix_predictability} asks how well one can forecast the average terminal reward of $m$ fresh continuations sampled after a history $\mathcal{H}_t$. The optimal mean-squared error $\mathcal{E}_{t,m}^\star$ measures this predictability using all information in $\mathcal{H}_t$, with target $\mathbb{E}[\bar R_{t,m}\mid\mathcal{H}_t]=V_t^\pi$. Because $\mathcal{H}_{t+1}$ extends $\mathcal{H}_t$ with one additional turn, prediction cannot worsen as more interaction is observed ($\mathcal{E}_{t+1,m}^\star\le \mathcal{E}_{t,m}^\star$), so prefix-level scoring is at least as informative as prompt-only scoring and strictly better once intermediate turns are available. This justifies scoring visited prefixes before Stage~2 continuation allocation. We attach the corresponding proof in Appendix~\ref{sec:prefix_predictability_proof}.

\subsection{A Shared Mixed-Outcome Principle for Roots and Prefixes}
\label{sec:mixed_reward_utility}

Given such root- and prefix-level predictions, the allocation target is the same at both scales: at a prompt root, the allocator asks whether fresh rollouts will form a diverse outcome group; at a visited prefix, it asks whether additional continuations from the same history can disagree with the factual suffix. This gives a single contrast-seeking rule: anchors near deterministic success or failure provide little contrast, while anchors with intermediate conditional success probability are more likely to expose both successful and failed descendants.

This prefix-level view explains how outcome-only rewards can still induce local credit. Shared-prefix branches keep the past fixed and vary only the future continuation, so opposite terminal outcomes become a local comparison for tree-aware optimizers. From a martingale view of conditional success, this remaining contrast has a simple form: each additional turn revises the forecast of terminal success, and the squared forecast revisions accumulate the remaining uncertainty along the suffix rather than only at the terminal reward.

\begin{proposition}[Prefix Uncertainty as Remaining Contrast Potential]
\label{prop:energy_identity}
Let $\mathcal{F}_t:=\sigma(\mathcal{H}_t)$, $Z_t:=V_t^\pi=\mathbb{E}_{\pi}[r(\mathcal{H}_T)\mid\mathcal{F}_t]$, and let $[Z]_{t:T}:=\sum_{s=t}^{T-1}(Z_{s+1}-Z_s)^2$ denote the quadratic variation of the conditional success forecast along the future rollout suffix. For any prefix $\mathcal{H}_t$,
\begin{equation}
  \mathbb{E}_{\pi}\left[[Z]_{t:T}\mid \mathcal{F}_t\right]
  = V_t^\pi\bigl(1-V_t^\pi\bigr).
  \label{eq:energy_identity}
\end{equation}
\end{proposition}

Proposition~\ref{prop:energy_identity} identifies the conditional expected quadratic variation with Bernoulli variance: $\mathbb{E}_{\pi}[[Z]_{t:T}\mid \mathcal{F}_t]=V_t^\pi(1-V_t^\pi)$. Thus $V_t^\pi(1-V_t^\pi)$ is not merely a static uncertainty score, but the expected accumulated movement in conditional success probability below the prefix, and therefore measures downstream contrast potential. Fig.~\ref{fig:contrast_allocation} confirms the corresponding empirical picture: prefix information improves group-success prediction, and a small subset of ranked anchors captures much of the total pair contrast. We prove Proposition~\ref{prop:energy_identity} in Appendix~\ref{sec:energy_identity_proof}.

\subsection{Contrast Values as Update Activation}
\label{sec:contrast_update_activation}

The allocation rules are chosen before the optimizer sees gradients, but they target the activation frontier where gradients become informative. Under binary rewards, a root or prefix anchor activates local gradient contrast only when its descendants expose both success and failure.

This observation connects contrast allocation to optimizer-visible gradients. Let $\mathcal{A}_{\mathrm{root}}=\{x_i\}_{i=1}^{B}$ and let $\mathcal{A}_{\mathrm{pref}}$ denote the local prefix-anchor set $\mathcal{A}_i$ for an active prompt; aggregation over prompts simply sums these local terms. For stage $q\in\{\mathrm{root},\mathrm{pref}\}$ and allocation $b$, we formulate the aggregate local policy-gradient contribution in general form as
\begin{equation}
  \begin{gathered}
  G_q(b)
  =
  \sum_{h\in\mathcal{A}_q}
  \mathbb{I}\{\mathsf{Act}(h,b_h)\}
  \widetilde G_h(b_h),\\
  \widetilde G_h(b_h)
  =
  \sum_{\mathcal{H}\in\mathcal{D}_h(b_h)}
  \omega_{\mathcal H}
  \nabla_\theta\log\pi_\theta(\mathcal{H}\mid h),
  \end{gathered}
  \label{eq:local_gradient_contribution}
\end{equation}
Here $\mathcal{D}_h(b_h)$ is the sampled terminal-descendant set below anchor $h$, $\omega_{\mathcal H}$ is the optimizer-dependent contrast weight, $\mathbb{I}\{\cdot\}$ is the indicator function, and $\mathsf{Act}(\cdot,\cdot)$ is the activation event: $\mathsf{Act}(h,b_h)$ holds when the sampled descendants below anchor $h$ under budget $b_h$ contain both outcomes.

\begin{proposition}[Activation allocation dominates uniform]
\label{prop:trace_dominance}
Let $b_q^\star$ be the TRACE allocation for stage $q\in\{\mathrm{root},\mathrm{pref}\}$, and let $b_q^u$ be any feasible uniform allocation at the same budget. Write $G_q^\star:=G_q(b_q^\star)$ and $G_q^u:=G_q(b_q^u)$. Assume normalized conditional gradient energy, meaning $\mathbb{E}[\|\widetilde G_h(b_h)\|^2\mid\mathsf{Act}(h,b_h),h]=c_q>0$ for all anchors in stage $q$. If the combined update also satisfies $\mathbb{E}\langle G_{\mathrm{root}}^\star,G_{\mathrm{pref}}^\star\rangle\ge\mathbb{E}\langle G_{\mathrm{root}}^u,G_{\mathrm{pref}}^u\rangle$, then
\begin{equation}
  \begin{gathered}
  \mathbb{E}\|G_q^\star\|^2
  \ge
  \mathbb{E}\|G_q^u\|^2,
  \quad q\in\{\mathrm{root},\mathrm{pref}\},\\
  \mathbb{E}\|G_{\mathrm{root}}^\star+G_{\mathrm{pref}}^\star\|^2
  \ge
  \mathbb{E}\|G_{\mathrm{root}}^u+G_{\mathrm{pref}}^u\|^2.
  \end{gathered}
  \label{eq:trace_dominates_uniform}
\end{equation}
Inequalities are strict when uniform is not optimal for the corresponding stage.
\end{proposition}

Proposition~\ref{prop:trace_dominance} suggests the main allocation consequence: the first line of Eq.~\eqref{eq:trace_dominates_uniform} shows that each stage separately strengthens the expected squared local gradient signal over uniform allocation, and the second line shows that their combination strengthens the full tree update when the two gradient sources do not cancel. The proof in Appendix~\ref{app:allocation_dominance} shows the factorization behind this statement: under standard outcome-reward updates, the local gradient contribution is zero unless opposite outcomes are observed below the same anchor, so the expected squared gradient norm equals an activation probability times a conditional gradient scale.

\section{TRACE}
\label{sec:method}

The previous section identifies the useful sampling targets: roots and prefixes whose descendants are likely to disagree in terminal reward. TRACE instantiates this principle as a global-to-local allocation procedure, summarized in Fig.~\ref{fig:trace_framework} and Algorithm~\ref{alg:trace}. It solves root allocation over a candidate prompt pool and local prefix expansion over visited prefix occurrences. The predictor $\tilde V_\psi(\mathcal H)$ estimates the same conditional success probability written as $V_t^\pi$ for a turn-indexed history $\mathcal H_t$ in Section~\ref{sec:prelim}. The resulting trees supervise this predictor through recursive targets and provide root-/prefix-level comparisons for tree-aware policy optimization.

\begin{figure*}[t]
  \centering
  \includegraphics[width=0.98\textwidth]{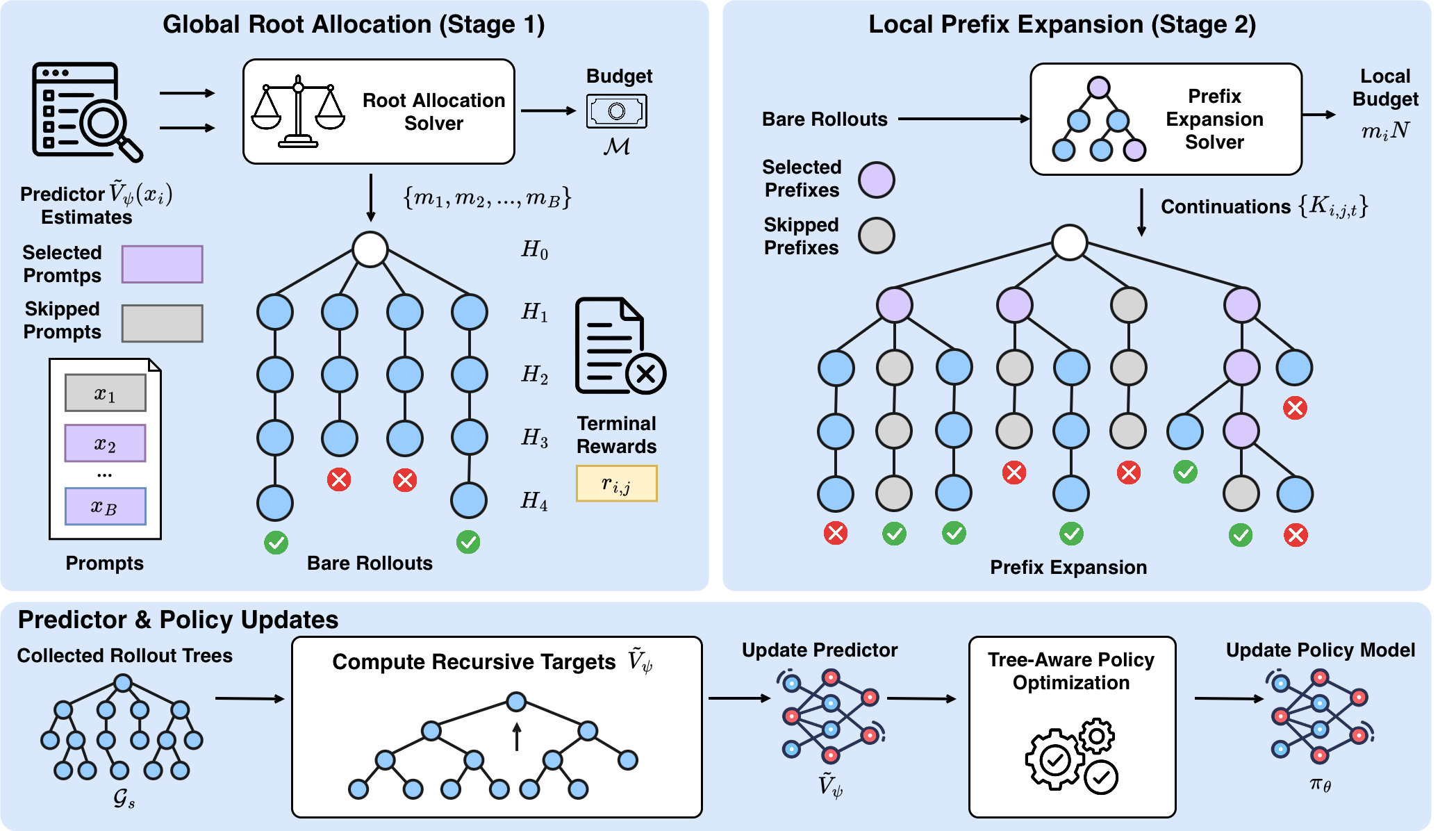}
  \caption{\textbf{Framework overview of TRACE.} A prefix value predictor scores prompt roots and visited prefixes, TRACE solves budgeted root allocation and prompt-local prefix expansion, and the resulting rollout trees provide recursive value targets and root-/prefix-level comparisons for tree-aware policy optimization.}
  \label{fig:trace_framework}
  \vspace{-5pt}
\end{figure*}

\subsection{Global Root Allocation}
\label{sec:bg_constrained_allocation}

At step $s$, TRACE allocates a root rollout budget over a candidate prompt set $\mathcal{B}_s=\{x_1,\ldots,x_B\}$, where $B=|\mathcal{B}_s|$ is the candidate batch size and each prompt receives a count in $\mathcal{M}:=\{0\}\cup\{2,\ldots,M\}$. Let $v_i:=\tilde V_\psi(x_i)$ be the predicted root success probability. For $m\ge2$, TRACE assigns
\begin{equation}
  V_{\mathrm{root}}(x_i,m)
  =
  1-v_i^m-(1-v_i)^m,
  \label{eq:root_esr_value}
\end{equation}
and sets $V_{\mathrm{root}}(x_i,0)=V_{\mathrm{root}}(x_i,1)=0$. This is the predicted probability that $m$ root rollouts for $x_i$ contain success and failure. Formally, root allocation solves the budget-constrained problem:
\begin{equation}
  \begin{gathered}
  \max_{m_1,\ldots,m_B}
  \sum_{i=1}^{B}V_{\mathrm{root}}(x_i,m_i)\\
  \text{s.t.},\;
  \sum_{i=1}^{B}m_i=M,\;
  m_i\in\mathcal{M}.
  \end{gathered}
  \label{eq:stage1_root_allocation}
\end{equation}
Setting $m_i=0$ skips the prompt. We exclude $m_i=1$ from $\mathcal{M}$ because typical tree-aware policy optimization is group-based and requires at least two root rollouts per prompt. Any $m_i\ge2$ activates the prompt and sets the number of bare rollouts generated below it.

\subsection{Local Prefix Expansion}

After active prompt $x_i$ finishes its $m_i$ bare rollouts, let $r_{i,j}\in\{0,1\}$ be the terminal reward of bare rollout $j$. Its prefix after turn $t$ is $\mathcal{H}_{i,j,t}$, with $\mathcal{H}_{i,j,0}=x_i$. Recall the prefix-anchor index set $\mathcal{A}_i$ defined in Eq.~\eqref{eq:prefix_anchor_set}, each index $(j,t)\in\mathcal{A}_i$ denotes a visited nonterminal prefix together with the factual terminal outcome already observed below it. We allocate a fixed local budget
\begin{equation}
  \sum_{(j,t)\in\mathcal{A}_i}K_{i,j,t}=m_iN,
\end{equation}
so Stage~2 does not wait for or compete with other prompts once prompt $x_i$ is ready.

Using the same predictor for the prefix success probability, assigning $k$ continuations to prefix occurrence $(j,t)$ has value
\begin{equation}
  V_{\mathrm{pref}}(i,j,t,k)
  :=
  1
  -\Big[
  r_{i,j}\tilde V_\psi(\mathcal{H}_{i,j,t})
  +(1-r_{i,j})(1-\tilde V_\psi(\mathcal{H}_{i,j,t}))
  \Big]^k,
  \label{eq:pref_utility_practical}
\end{equation}
namely the probability that at least one new continuation flips the observed reward. For each active prompt $x_i$, local prefix expansion solves
\begin{equation}
  \begin{gathered}
  \max_{\{K_{i,j,t}\}_{(j,t)\in\mathcal{A}_i}}
  \sum_{(j,t)\in\mathcal{A}_i}
  V_{\mathrm{pref}}(i,j,t,K_{i,j,t})\\
  \text{s.t.}\;
  \sum_{(j,t)\in\mathcal{A}_i}K_{i,j,t}=m_iN.
  \end{gathered}
  \label{eq:stage2_prefix_allocation}
\end{equation}
Both allocation problems are solved by efficient dynamic programming tools, whose cost is negligible relative to rollout generation. Expanding $(j,t)$ fixes the prefix $\mathcal{H}_{i,j,t}$ and resamples a new continuation after that prefix. By design, terminal leaves are meaningless to expand.

This factorization is intentionally system-aware: a fully global prefix allocator over $\bigcup_i\mathcal{A}_i$ would wait for all active prompts to finish bare rollouts before launching any continuation. TRACE only requires prompt-level completion. Since the $m_i$ bare rollouts for prompt $x_i$ are co-located on the same worker, their generation does not suffer from intra-prompt long-tail waiting. Consequently, the expansion for prompt $x_i$ can be enqueued immediately upon their return, bypassing the inter-prompt waiting caused by other prompts distributed across different workers.

\subsection{Online Prefix Value Estimation}
\label{sec:value_model}

The allocation objectives require the conditional success probability at both prompt roots and visited prefixes. We learn a shared generalizable predictor that estimates it from prefix histories at both levels.
\begin{equation}
  \tilde V_\psi:\;\mathcal{H}_t\;\mapsto\;\tilde V_\psi(\mathcal{H}_t)\in[0,1],
  \label{eq:predictor_interface}
\end{equation}
which provides $\tilde V_\psi(x_i)$ for Eq.~\eqref{eq:root_esr_value} and $\tilde V_\psi(\mathcal{H}_{i,j,t})$ for Eq.~\eqref{eq:pref_utility_practical}. It is used only by the allocator, not by the downstream policy optimizer.

\paragraph{Recursive tree-backed targets.}
After a rollout tree is collected, its terminal leaves provide binary rewards. For any node $y$, let $\mathcal{C}(y)$ be its children and let $n_y$ be the number of executed terminal descendants below it:
\[
  n_y =
  \begin{cases}
    1, & y \text{ is a terminal leaf},\\
    \sum_{c\in\mathcal{C}(y)}n_c, & \text{otherwise},
  \end{cases}
\]
We set $\widehat V(y)=r_y$ at terminal leaves and compute internal targets by a bottom-up empirical average:
\begin{equation}
  \widehat V(y)
  =
  \frac{1}{n_y}
  \sum_{c\in\mathcal{C}(y)}
  n_c\widehat V(c).
  \label{eq:node_value_recursive}
\end{equation}
Thus $\widehat V(y)$ is the empirical success rate of terminal descendants below $y$. We train $\tilde V_\psi$ by mean-squared regression over a supervision set $\mathcal{S}$ of roots and informative internal nodes:
\begin{equation}
  \mathcal{L}_{\mathrm{value}}
  :=
  \frac{1}{|\mathcal{S}|}
  \sum_{y\in\mathcal{S}}
  \left(\tilde V_\psi(y)-\widehat V(y)\right)^2.
  \label{eq:value_loss}
\end{equation}
The exact supervision set and lightweight predictor implementation details are reported in Appendix~\ref{app:value_impl}.

\begin{algorithm}[t]
  \caption{TRACE rollout allocation}
  \label{alg:trace}
  \KwInput{
    Candidate pool sequence $\{\mathcal{B}_s\}_{s=1}^{S}$ with $|\mathcal{B}_s|=B$;
    initial policy $\pi_\theta$;
    predictor $\tilde V_\psi$;
    root budget $M$;
    expansion factor $N$.
  }
  \KwOutput{Finetuned LLM policy $\pi_\theta$.}
  \For{$s=1$ \KwTo $S$}{
    Initialize collected tree set $\mathcal{G}_s\leftarrow\emptyset$\;
    \tcp{\textcolor{blue}{Global root allocation}}
    Estimate $\tilde V_\psi(x_i)$ for all $x_i\in\mathcal{B}_s$\;
    Solve Eq.~\eqref{eq:stage1_root_allocation} for root counts $\{m_i\}_{i=1}^{B}$\;
    \ForEach{prompt $x_i$ with $m_i>0$}{
      Generate $m_i$ bare rollouts from $x_i$ and collect terminal rewards\;
      Build the prefix-anchor index set $\mathcal{A}_i$\;
      Estimate $\tilde V_\psi(\mathcal{H}_{i,j,t})$ for each $(j,t)\in\mathcal{A}_i$\;
      \tcp{\textcolor{blue}{Local prefix expansion}}
      Solve Eq.~\eqref{eq:stage2_prefix_allocation} for continuation counts $\{K_{i,j,t}\}$\;
      \ForEach{$(j,t)$ with $K_{i,j,t}>0$}{
        Generate $K_{i,j,t}$ continuations from prefix $\mathcal{H}_{i,j,t}$\;
      }
      Add the resulting prompt-local rollout trees to $\mathcal{G}_s$\;
    }
    \tcp{\textcolor{blue}{Predictor and policy updates}}
    Compute recursive targets $\widehat V$ on $\mathcal{G}_s$ and update $\tilde V_\psi$ by Eq.~\eqref{eq:value_loss}\;
    Update $\pi_\theta$ with any tree-aware optimizer using $\mathcal{G}_s$\;
  }
\end{algorithm}

\subsection{Tree-Aware Policy Optimization}
\label{sec:policy_optimization}

TRACE passes the completed rollout trees to a tree-aware policy optimizer as in Algorithm~\ref{alg:trace}. This optimizer can use any prefix-level credit rule that propagates leaf rewards through the tree, including process-reward backups or group-relative backups. In our experiments, we use the concrete instance described in Appendix~\ref{app:treegrpo_instance}.

\section{Experiments}
\label{sec:experiments}

\subsection{Experimental Setup}
\label{sec:setup}

\paragraph{Datasets and models.} We evaluate TRACE on three representative multi-turn settings. (1) \textbf{Mathematical Reasoning}: trains on the DeepScaler corpus~\citep{luo2025deepscaler} with a Python interpreter and evaluates on AIME24, AMC23, MATH500~\citep{lightman2024let}, MinervaMath (Minerva)~\citep{lewkowycz2022solving}, OlympiadBench (Olympiad)~\citep{he2024olympiadbench}, and three out-of-distribution benchmarks: MMLU-Pro~\citep{wang2024mmlu}, ARC-c~\citep{clark2018think}, and GPQA-diamond (GPQA)~\citep{rein2023gpqa}. (2) \textbf{Multi-Hop QA}: trains on HotpotQA (Hotpot)'s training split~\citep{yang2018hotpotqa} and evaluates on its validation split, 2WikiMultiHopQA (2Wiki)~\citep{ho2020constructing}, MuSiQue (Musiq)~\citep{trivedi2022musique}, and Bamboogle (Bamb)~\citep{press2023measuring} with a local E5 retrieval server~\citep{wang2022text} built over a Wikipedia dump~\citep{karpukhin2020dense}. (3) \textbf{Function Calling}: trains on 80\% of the BFCL v4~\citep{patil2025berkeley} multi-turn split and evaluates on the held-out 20\% test portion across the base (Base), long-context (Long), missing-function (Miss-Func), and missing-parameter (Miss-Param) subsets. Main experiments use Qwen3-8B and Qwen3-14B policy backbones~\citep{yang2025qwen3}. Appendix~\ref{sec:additional_results} further reports results using Llama-3.2-3B-Instruct~\citep{grattafiori2024llama} as an additional backbone.

\paragraph{Baselines.} We compare TRACE against four baselines. (1) \textbf{ReAct}~\citep{yao2022react}: evaluates the base model in the same multi-turn agent and tool scaffold. (2) \textbf{GRPO}~\citep{shao2024deepseekmath}: samples prompts randomly at the prompt level. (3) \textbf{PCL}~\citep{gao2025prompt}: a pure CoT prompt-selection baseline that actively filters prompts by difficulty. (4) \textbf{TreePO}: adds tree-structured rollouts and tree-aware updates on top of GRPO, but still selects turn-level continuations randomly, as instantiated in Appendix~\ref{app:treegrpo_instance}. All methods, including TRACE, use the same rollout-budget accounting described in Appendix~\ref{app:init_expand}, so gains reflect allocation rather than additional samples.

\paragraph{Metrics.} We report final-answer accuracy for \textbf{Mathematical Reasoning}, HotpotQA-style exact match and token-level F1 with partial credit at F1 $\ge 0.3$ for \textbf{Multi-Hop QA}, and official BFCL success rate for \textbf{Function Calling}, averaging each benchmark over a task-specific number of sampled trajectories and reporting the average across benchmarks within each setting. Appendix~\ref{sec:additional_results} reports allocation and predictor diagnostics, while Appendix~\ref{sec:experimental_details} gives full evaluation protocols, avg@$k$ multiplicities, dataset splits, model settings, and rollout details.

\subsection{Main Results}
\label{sec:main_results}

\begin{figure*}[t]
  \centering
  \begin{subfigure}[t]{0.32\textwidth}
    \centering
    \includegraphics[width=\linewidth]{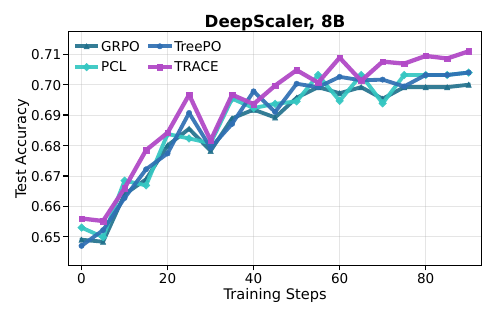}
  \end{subfigure}
  \begin{subfigure}[t]{0.32\textwidth}
    \centering
    \includegraphics[width=\linewidth]{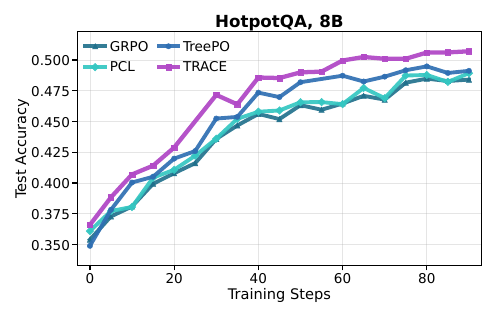}
  \end{subfigure}
  \begin{subfigure}[t]{0.32\textwidth}
    \centering
    \includegraphics[width=\linewidth]{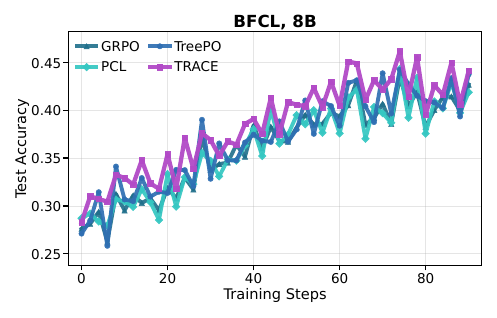}
  \end{subfigure}
  \begin{subfigure}[t]{0.32\textwidth}
    \centering
    \includegraphics[width=\linewidth]{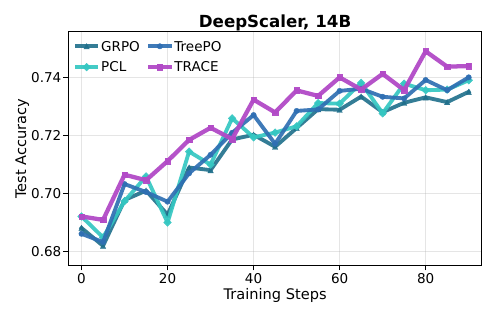}
  \end{subfigure}
  \begin{subfigure}[t]{0.32\textwidth}
    \centering
    \includegraphics[width=\linewidth]{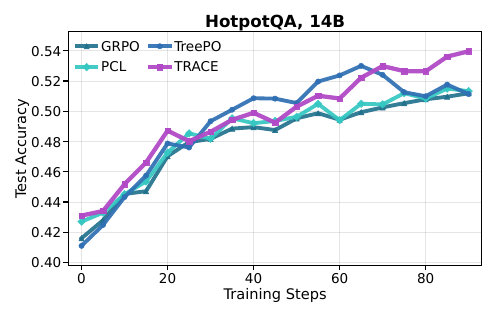}
  \end{subfigure}
  \begin{subfigure}[t]{0.32\textwidth}
    \centering
    \includegraphics[width=\linewidth]{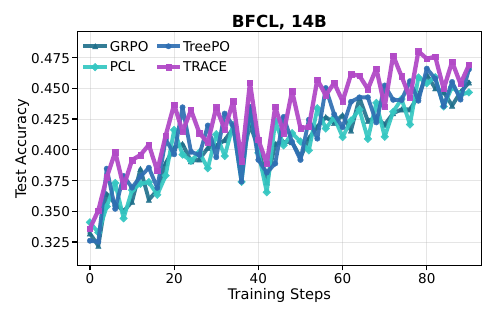}
  \end{subfigure}
  \caption{\textbf{Test accuracy during training.} The six panels cover Mathematical Reasoning (DeepScaler), Multi-Hop QA (HotpotQA), and Function Calling (BFCL v4) with Qwen3-8B (top) and Qwen3-14B (bottom). We compare GRPO, PCL, random TreePO allocation, and TRACE under the same rollout-budget setting. Higher curves indicate stronger final policies under identical sampling budgets.}
  \label{fig:hotpotqa_accuracy}
\end{figure*}

\begin{figure*}[t]
  \centering
  \begin{subfigure}[t]{0.32\textwidth}
    \centering
    \includegraphics[width=\linewidth]{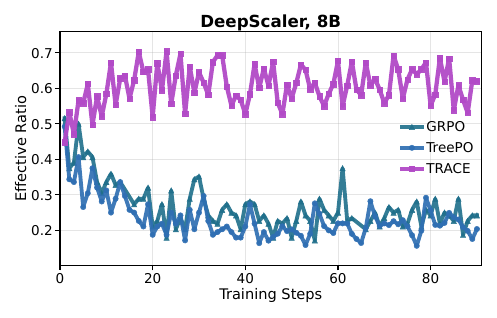}
  \end{subfigure}
  \begin{subfigure}[t]{0.32\textwidth}
    \centering
    \includegraphics[width=\linewidth]{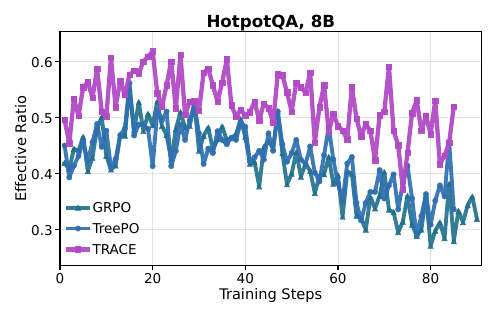}
  \end{subfigure}
  \begin{subfigure}[t]{0.32\textwidth}
    \centering
    \includegraphics[width=\linewidth]{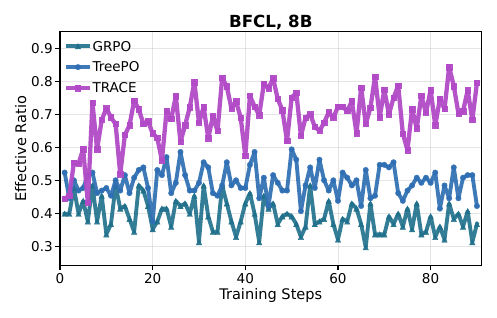}
  \end{subfigure}
  \vspace{+2pt}
  \begin{subfigure}[t]{0.32\textwidth}
    \centering
    \includegraphics[width=\linewidth]{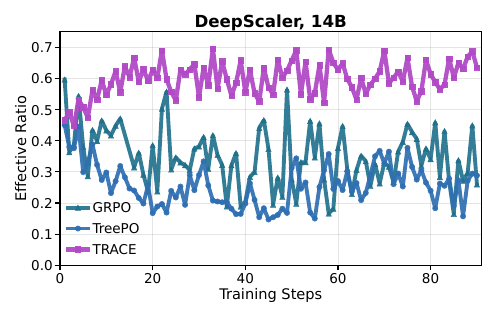}
  \end{subfigure}
  \begin{subfigure}[t]{0.32\textwidth}
    \centering
    \includegraphics[width=\linewidth]{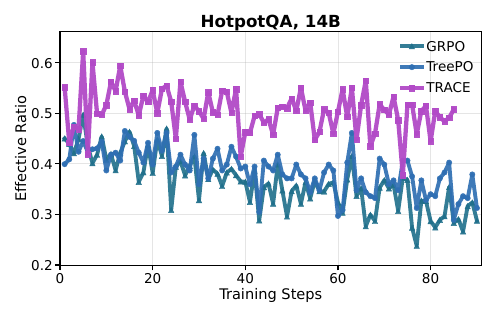}
  \end{subfigure}
  \begin{subfigure}[t]{0.32\textwidth}
    \centering
    \includegraphics[width=\linewidth]{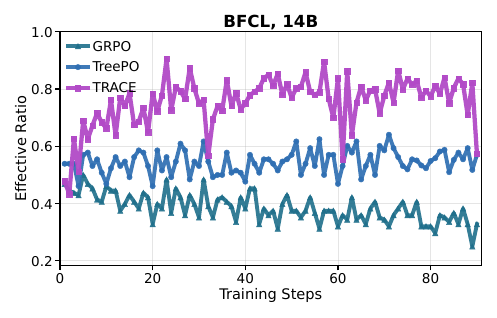}
  \end{subfigure}
  \caption{\textbf{Effective ratio during training.} The six panels cover Mathematical Reasoning (DeepScaler), Multi-Hop QA (HotpotQA), and Function Calling (BFCL v4) with Qwen3-8B (top) and Qwen3-14B (bottom). The panels compare the fraction of contrastive samples selected by each allocation strategy. Higher values indicate more non-degenerate reward groups per update.}
  \label{fig:app_prompt_effective_ratio}
\end{figure*}

\subsubsection{Improved Accuracy under the Same Budget}

Figure~\ref{fig:hotpotqa_accuracy} shows the training curves across all three settings and two model sizes. Overall, TRACE improves average performance across Mathematical Reasoning, Multi-Hop QA, and Function Calling for both Qwen3-8B and Qwen3-14B, consistently yielding higher evaluation curves than GRPO, PCL, and TreePO under the same rollout budget. On Mathematical Reasoning (DeepScaler), it improves the in-distribution average over GRPO from 70.0 to 71.1 for Qwen3-8B and from 73.5 to 74.9 for Qwen3-14B; similar gains appear on Multi-Hop QA and Function Calling. The TreePO comparison is especially revealing because both methods share tree rollouts and tree-aware updates, yet TRACE allocates branches to roots and prefixes predicted to be more informative rather than sampling them randomly. Full per-benchmark breakdowns are reported in Appendix~\ref{sec:app_detailed_results}.

\subsubsection{More Informative Rollouts under the Same Budget}

Figure~\ref{fig:app_prompt_effective_ratio} reports the effective ratio during training. At the prompt level, this metric is the within-batch fraction of prompts whose rollout trees contain terminal leaves with both successful and failed outcomes. Under a matched rollout budget, TRACE consistently achieves a higher training-time average effective ratio than GRPO, PCL, and TreePO across all three settings and both model scales. On Mathematical Reasoning (DeepScaler), it raises the average over GRPO from 26.8\% to 60.6\% for Qwen3-8B and from 34.7\% to 59.7\% for Qwen3-14B; similar gains appear on Multi-Hop QA and Function Calling. By directing rollout budget toward contrast-rich roots and prefixes rather than saturated or low-variance anchors, TRACE increases the frequency of non-degenerate terminal groups and amplifies the group-relative policy-update signal per rollout unit. The full metric definition appears in Appendix~\ref{sec:app_detailed_results}.

\subsubsection{Reliable Difficulty Prediction}

Figures~\ref{fig:app_prompt_spearman} and~\ref{fig:app_node_spearman} show that the lightweight predictor learns usable difficulty rankings at both prompt and node levels. The prompt-level signal is expected, since root outcomes provide the densest supervision. More importantly, the signal transfers to internal prefixes: even when training is dominated by prompt-level supervision, the same predictor still ranks node difficulty reliably enough to guide turn-level branching. This suggests that prefix histories carry a learnable notion of local uncertainty, not just noisy fragments of full trajectories. Appendix~\ref{sec:additional_results} further reports the Spearman definition and per-domain trends.

\subsection{Additional Analysis}

\subsubsection{Ablation Studies}

We isolate the two allocation stages on Qwen3-8B Multi-Hop QA (HotpotQA) by replacing each active allocator with uniform allocation while keeping the tree budget and optimizer fixed. Table~\ref{tab:stage_ablation} shows that both stages help and their gains stack: root allocation selects prompts likely to become non-degenerate, while prefix allocation spends continuation budget on prefixes where additional descendants can still reveal contrast.

\begin{table}[t]
    \centering
    \caption{\textbf{Ablation of active allocation stages on Qwen3-8B Multi-Hop QA (HotpotQA).} Stage~1 allocates prompt-root budget; Stage~2 allocates continuation budget over visited prefixes.}
    \vspace{-5pt}
    \resizebox{0.5\columnwidth}{!}{
    \begin{tabular}{cccc}
        \toprule
        \textbf{Stage 1} & \textbf{Stage 2} & \textbf{Avg. Acc.} & \textbf{Avg. Eff. Ratio} \\
        \midrule
        Uniform & Uniform &49.5 &42.8 \\
        Active & Uniform &49.8 &49.1 \\
        Uniform & Active &50.0 &47.3 \\
        Active & Active &\textbf{50.6} &\textbf{52.3} \\
        \bottomrule
    \end{tabular}}
    \label{tab:stage_ablation}
\end{table}

\subsubsection{Compatibility with Different Budgets}

We also vary the tree budget on Qwen3-8B Multi-Hop QA (HotpotQA). Since the expected rollout budget is $M(1+N/2)$, $(512,2)$ uses budget 1024, while $(512,6)$ and $(1024,2)$ both use budget 2048. Table~\ref{tab:hotpotqa_budget_sensitivity} shows that TRACE improves TreePO at every budget, lifting the effective ratio by about 9.5--10.2 points and gaining roughly one accuracy point. It also shows that budget shape matters, not only total size: at budget 2048, broader root coverage with $(1024,2)$ is stronger than deeper continuation sampling from 512 roots. The bottleneck is therefore not just rollout count, but whether the budget reaches states where rewards can form contrast.

\begin{table}[t]
    \centering
    \caption{\textbf{Compatibility with different budgets on Qwen3-8B Multi-Hop QA (HotpotQA).} We vary the root budget $M$ and continuation budget $N$ and compare random TreePO allocation with TRACE.}
    \vspace{-5pt}
    \resizebox{0.5\columnwidth}{!}{
    \begin{tabular}{ccccc}
        \toprule
        \textbf{$M$} & \textbf{$N$} & \textbf{Method} & \textbf{Avg. Acc.} & \textbf{Avg. Eff. Ratio} \\
        \midrule
        512 & 2 & TreePO & 48.8 & 32.2 \\
        512 & 2 & TRACE & 49.7 & 42.4 \\
        512 & 6 & TreePO & 49.4 & 37.7 \\
        512 & 6 & TRACE & 50.3 & 47.8 \\
        1024 & 2 & TreePO & 49.5 & 42.8 \\
        1024 & 2 & TRACE & \textbf{50.6} & \textbf{52.3} \\
        \bottomrule
    \end{tabular}}
    \label{tab:hotpotqa_budget_sensitivity}
\end{table}

\section{Conclusion}
\label{sec:conclusion}

This work formulates mixed-reward contrast allocation as a unified view of sample-efficient RLVR for multi-turn agents. We show that prompt filtering, rollout-count allocation, and prefix branching are all budget decisions over rollout-tree anchors, and present TRACE to allocate budget guided by a shared generalizable predictor of conditional success probability. TRACE directs budget toward anchors whose descendants are likely to expose mixed terminal outcomes, strengthening the group-relative update signal and implicit stepwise preference pairs under outcome-only rewards. Experiments on Mathematical Reasoning, Multi-Hop QA, and Function Calling confirm improved accuracy and rollout informativeness under a matched budget. Moreover, this contrast-allocation view reframes agentic RLVR sampling as a question of where to branch within the rollout tree, not only how many independent rollouts to draw.

\section*{Limitations}

TRACE is mainly developed for outcome-only RLVR, and tasks without clear terminal verification may therefore require revisiting the mixed-outcome contrast objective. For allocation, higher prediction accuracy can lead to better budget placement, while predictor instantiations remain orthogonal to the TRACE framework. In this work, we instantiate the predictor in a vanilla setting, and stronger alternatives for different agentic tasks remain future work. Empirically, our evaluation is mainly conducted on Mathematical Reasoning, Multi-Hop QA, and Function Calling with Qwen3-8B and Qwen3-14B, leaving more complex and non-stationary agentic scenarios for future exploration.





\bibliographystyle{iclr2025_conference}
\bibliography{custom}

\newpage

\appendix

\section*{Appendix Overview}

This appendix provides supplementary material for TRACE, including related work, discussion, extended experimental results, theoretical proof, experimental details, and representative data examples.
The appendix is organized as follows:

\begin{itemize}[leftmargin=10pt]
  \item \textbf{Appendix~\ref{sec:related} (Related Work):}
  reviews prior work on RLVR and agentic RLVR, sample-efficient prompt selection and allocation, tree search for LLMs, and credit assignment under sparse rewards.

  \item \textbf{Appendix~\ref{sec:discussion} (Discussion):}
  discusses allocation as an orthogonal layer, the axes of training-time sampling in Table~\ref{tab:method_compare}, and root-to-prefix predictor generalization.

  \item \textbf{Appendix~\ref{sec:additional_results} (Extended Experimental Results):}
  reports the full benchmark tables and diagnostic analyses, including effective-ratio behavior, prefix-predictor rank correlation, transfer to Llama-3.2-3B-Instruct, computational overhead, and TRACE allocation behavior.

  \item \textbf{Appendix~\ref{sec:proof} (Theoretical Proof):}
  gives detailed proofs for prefix predictability, the energy identity connecting prefix uncertainty to reward variance, and the TRACE dominance result.

  \item \textbf{Appendix~\ref{sec:experimental_details} (Experimental Details):}
  describes the tasks, models, training settings, policy optimizers and TreePO instances, rollout construction and budget accounting, and prefix-predictor supervision and implementation.

  \item \textbf{Appendix~\ref{appsec:dataexample} (Data Examples):}
  presents representative real data examples from the three training domains, including the system prompts and tool schemas seen by the agents.
\end{itemize}

\section{Related Work}
\label{sec:related}

\paragraph{RLVR and Agentic RLVR.}
Reinforcement learning with verifiable rewards (RLVR) has become a dominant post-training paradigm for LLMs on tasks with automatically checkable outcomes~\citep{jaech2024openai,guo2025deepseek,team2025kimi}. Proximal Policy Optimization (PPO)~\citep{schulman2017proximal} remains a general policy-gradient backbone, while Group Relative Policy Optimization (GRPO)~\citep{shao2024deepseekmath} is especially attractive for RLVR because it removes the expensive value model and estimates advantages from groups of sampled rollouts. Much of RLVR still lives in a flat regime where a prompt is sampled, complete responses are generated, and a terminal reward drives the update. Agentic environments fracture this abstraction. A rollout is a chain of thought-action-observation turns~\citep{yao2022react}, and intermediate prefixes such as search queries, tool calls, and partial solution states can alter the entire downstream trajectory. TRACE uses this structure to turn rollout budget from a scalar sampling count into a structural allocation decision across roots and prefixes.

\paragraph{Sample Efficient RLVR.}
Sample-efficient RLVR improves policy updates by deciding which examples deserve rollout budget. Offline curation filters prompts by difficulty, diversity, length, or estimated solution quality~\citep{ye2025limo,li2025limr,wen2025light,hu2025open,yang2024qwen2math,fatemi2025concise,wang2025oneexample}, reducing training cost but adding preprocessing and remaining static as the policy evolves~\citep{qu2025prompt,gao2025prompt,mao2026rlvr}. Online prompt selection addresses this limitation by adapting to the current model, either through rollout-based filtering of uninformative prompts~\citep{yu2025dapo,liu2025prorl,cui2025process,meng2025mm,bae2026online,xu2026prune,zhang2026train} or through predictors that estimate prompt difficulty~\citep{zhang2025srpo,zheng2025act,gao2025prompt,qu2025prompt,zeng2025cures,chen2025self,mao2026dynamics,qu2026small,zou2025utility,zhang2026v_,hu2025vade,shen2025bots}. A complementary prompt-level allocation line asks how many samples each selected prompt should receive~\citep{li2025knapsack,yao2026coba}. This line is effective for dynamic rollout accounting, but it typically relies on pre-calculated per-prompt accuracy and has limited ability to generalize to unseen prompts without explicit rollouts. TRACE treats filtering, rollout counts, and prefix branching as boundary cases of one allocation operation, using a shared generalizable predictor to score candidate anchors from prompt roots to internal prefixes and an allocator to assign budget according to predicted conditional success probability.

\paragraph{Tree Search for LLMs.}
Tree search lets LLMs branch over interaction prefixes and compare with alternative continuations, rather than sampling a single complete trajectory. Test-time methods such as Tree-of-Thought~\citep{yao2023tree,long2023large,snell2024scaling,koh2024tree,zhou2023language} and MCTS-style decoding~\citep{xin2025deepseek} spend additional computation at test time to search for a stronger answer for each prompt. Offline variants sample or score reasoning trees and distill the induced comparisons into SFT or preference-learning data~\citep{feng2023alphazero,he2024advancing,xie2024monte,zhang2025process,lai2024step,li2025iterative}. Online tree-based RL brings branching into training, using tree rollouts for value estimation or finer credit assignment~\citep{zhang2024rest,hou2025treerl,ji2025tree,yang2025treerpo,kazemnejad2024vineppo,guo2025segment,dong2025agentic,zhao2026training}. TRACE is closest to this online lineage, but repurposes the tree. It branches over agent interaction prefixes rather than tokens or generic reasoning steps, and uses the tree not as an inference solver, but as an allocation substrate for discovering reward contrast under a fixed rollout budget.

\paragraph{Credit Assignment under Sparse Rewards.}
Sparse outcome rewards are coarse signals for multi-turn agents because one terminal success or failure can be assigned back to many earlier decisions while intermediate decisions can carry sharply different responsibility for the final result~\citep{feng2025group}. Classical RL attacks this with exploration bonuses, replay, goal relabeling, reward shaping, or value-based credit estimates~\citep{pathak2017curiosity,andrychowicz2017hindsight,ng1999policy,arjona2019rudder}. In LLM agents, process reward models (PRMs) can generate step- or token-level feedback, but dense process rewards are costly to annotate and brittle across open-ended tool-use settings. Instead, tree rollouts already provide a form of implicit credit assignment by comparing continuations that share a prefix~\citep{ji2025tree,zhao2026training}. TRACE strengthens this mechanism without explicit process supervision by allocating more branches to prefixes whose alternative continuations are likely to diverge in outcome, turning sparse final rewards into sharper local preference signals.

\begin{table*}[t]
  \centering
  \caption{\textbf{Qualitative comparison of training-time allocation methods.} GRPO and PCL operate over flat rollouts, TreePO introduces a tree substrate with random branching, and TRACE uses learned allocation at both prompt roots and visited prefixes.}
  \vspace{-5pt}
  \renewcommand\arraystretch{1.15}
  \resizebox{0.86\textwidth}{!}{
  \setlength{\tabcolsep}{5.0mm}{
  \begin{tabular}{lccccc}
    \toprule
    \multirow{2}{*}{\textbf{Method}} &
    \textbf{Prompt} &
    \textbf{Root} &
    \textbf{Prefix} &
    \textbf{Tree-Structured} &
    \textbf{Learned} \\
    &
    \textbf{Selection} &
    \textbf{Budget} &
    \textbf{Expansion} &
    \textbf{Update} &
    \textbf{Allocator} \\
    \midrule
    \textbf{GRPO} & Random & Uniform & \(\times\) & \(\times\) & \(\times\) \\
    \textbf{PCL} & Predictive & Fixed & \(\times\) & \(\times\) & Root only \\
    \textbf{TreePO} & Random & Uniform & Random & \(\checkmark\) & \(\times\) \\
    \midrule
    \rowcolor{gray!20}\textbf{TRACE} & Predictive & Adaptive & Predictive & \(\checkmark\) & Root + prefix \\
    \bottomrule
  \end{tabular}}}
  \label{tab:method_compare}
\end{table*}

\section{Discussion}
\label{sec:discussion}

\subsection{Allocation as an Orthogonal Layer}

TRACE separates rollout acquisition from policy optimization. Tree-structured objectives such as TreeRPO~\citep{yang2025treerpo} and Tree-GRPO~\citep{ji2025tree} define how an existing tree becomes gradients through branch comparisons, value targets, or group-relative advantages. TRACE operates upstream by deciding which anchors deserve additional descendants before the optimizer sees the tree. Thus the same policy objective can receive a contrast-poor or contrast-rich tree under the same rollout budget.

The predictor and TreePO instance used in our experiments are practical instantiations rather than primary technical commitments. TRACE only requires a predictor that ranks anchors by expected reward contrast and a tree-structured optimizer that consumes the resulting rollout tree. Better predictors should improve budget placement, but the framework is not tied to a specific predictor architecture or tree-policy objective.

\subsection{Axes of Training-Time Sampling}

Table~\ref{tab:method_compare} organizes training-time sampling methods by the allocation axes they make controllable. GRPO exposes only a flat sampling axis, PCL turns prompt admission into a learned decision, and TreePO exposes a tree substrate while leaving the tree acquisition policy random. The missing axis is not another optimizer choice, but a structural sampling decision inside the rollout itself. Once a trajectory has interacted with a search engine, a Python interpreter, or an API environment, its visited prefixes become decision states with heterogeneous downstream uncertainty. Some have already collapsed into near-certain success or failure, while others remain poised at points where an alternative query, tool call, or reasoning continuation can alter the final reward. TRACE treats prompt roots and such prefixes as anchors in a single allocation geometry: a lightweight generalizable predictor scores where contrast is still latent, and the allocator decides whether to skip a root, widen it, or branch from an unresolved prefix.

\begin{table*}[t]
  \centering
  \caption{\textbf{Performance comparison on Mathematical Reasoning.} We train on the DeepScaler corpus and evaluate Qwen3-8B and Qwen3-14B on four in-distribution mathematical reasoning benchmarks and three out-of-distribution benchmarks. TRACE consistently outperforms competitive baselines across all metrics. The best results are indicated in \textbf{bold}.}
  \vspace{-5pt}
  \renewcommand\arraystretch{1.0}
  \resizebox{0.9\textwidth}{!}{
  \setlength{\tabcolsep}{1mm}{
  \begin{tabular}{cccccccc|cccc}
      \toprule
      \multirow{2}{*}{\textbf{Model}} &\multirow{2}{*}{\textbf{Method}} &\multicolumn{6}{c}{\textbf{In-distribution}} &\multicolumn{4}{c}{\textbf{Out-of-distribution}} \\
      \cmidrule(lr){3-8} \cmidrule(lr){9-12}
      & &\textbf{AIME24} &\textbf{AMC23} &\textbf{MATH500} &\textbf{Minerva} &\textbf{Olympiad} &\textbf{Avg} &\textbf{MMLU-Pro} &\textbf{ARC-c} &\textbf{GPQA} &\textbf{Avg} \\
      \midrule
      \multirow{5}{*}{\rotatebox{90}{\textbf{Qwen3-8B}}}
          &\textbf{ReAct} &52.3 &87.0 &89.8 &37.8 &58.6 &65.1 &68.5 &95.3 &52.4 &72.1 \\
          &\textbf{GRPO} &63.6 &91.0 &92.3 &40.2 &62.8 &70.0 &72.7 &95.3 &55.9 &74.6 \\
          &\textbf{PCL} &64.0 &91.5 &92.8 &40.7 &63.0 &70.4 &72.9 &95.4 &56.2 &74.8 \\
          &\textbf{TreePO} &64.3 &91.3 &92.6 &40.9 &63.1 &70.4 &73.0 &95.2 &56.4 &74.9 \\
          \cmidrule{2-12}
          &\cellcolor{gray!20}\textbf{TRACE} &\cellcolor{gray!20}63.9 &\cellcolor{gray!20}91.2 &\cellcolor{gray!20}93.4 &\cellcolor{gray!20}41.2 &\cellcolor{gray!20}65.8 &\cellcolor{gray!20}\textbf{71.1} &\cellcolor{gray!20}73.4 &\cellcolor{gray!20}95.7 &\cellcolor{gray!20}56.9 &\cellcolor{gray!20}\textbf{75.3} \\
      \midrule
      \multirow{5}{*}{\rotatebox{90}{\textbf{Qwen3-14B}}}
          &\textbf{ReAct} &61.5 &89.4 &91.8 &40.0 &62.3 &69.0 &75.1 &96.2 &59.2 &76.8 \\
          &\textbf{GRPO} &65.6 &94.3 &94.2 &45.2 &68.4 &73.5 &75.9 &96.1 &59.4 &77.1 \\
          &\textbf{PCL} &66.2 &95.1 &95.0 &45.1 &67.6 &73.9 &76.1 &96.2 &59.7 &77.3 \\
          &\textbf{TreePO} &66.4 &95.0 &94.8 &45.4 &67.8 &74.0 &76.0 &96.4 &59.8 &77.4 \\
          \cmidrule{2-12}
          &\cellcolor{gray!20}\textbf{TRACE} &\cellcolor{gray!20}66.1 &\cellcolor{gray!20}94.7 &\cellcolor{gray!20}94.9 &\cellcolor{gray!20}47.3 &\cellcolor{gray!20}71.5 &\cellcolor{gray!20}\textbf{74.9} &\cellcolor{gray!20}76.5 &\cellcolor{gray!20}96.5 &\cellcolor{gray!20}60.4 &\cellcolor{gray!20}\textbf{77.8} \\
      \bottomrule
  \end{tabular}}}
  \label{tab:deepscaler_results}
\end{table*}

\begin{table*}[t]
  \centering
  \caption{\textbf{Performance comparison on Multi-Hop QA and Function Calling.} We train on HotpotQA and BFCL v4, and evaluate Qwen3-8B and Qwen3-14B on four multi-hop reasoning benchmarks and four tool-calling scenarios, respectively. TRACE consistently outperforms competitive baselines across all metrics. The best results are indicated in \textbf{bold}.}
  \vspace{-5pt}
  \renewcommand\arraystretch{1.0}
  \resizebox{0.8\textwidth}{!}{
  \setlength{\tabcolsep}{1mm}{
  \begin{tabular}{ccccccc|ccccc}
      \toprule
      \multirow{2}{*}{\textbf{Model}} &\multirow{2}{*}{\textbf{Method}} &\multicolumn{5}{c}{\textbf{Multi-Hop QA}} &\multicolumn{5}{c}{\textbf{Function Calling}} \\
      \cmidrule(lr){3-7} \cmidrule(lr){8-12}
      & &\textbf{Hotpot} &\textbf{2Wiki} &\textbf{Musiq} &\textbf{Bamb} &\textbf{Avg} &\textbf{Base} &\textbf{Miss-Func} &\textbf{Miss-Param} &\textbf{Long} &\textbf{Avg} \\
      \midrule
      \multirow{5}{*}{\rotatebox{90}{\textbf{Qwen3-8B}}}
          &\textbf{ReAct} &41.9 &37.6 &18.4 &44.7 &35.7 &36.7 &21.5 &31.1 &22.3 &28.0 \\
          &\textbf{GRPO} &53.2 &48.7 &27.4 &64.7 &48.5 &58.5 &31.0 &45.8 &38.7 &43.5 \\
          &\textbf{PCL} &53.6 &49.2 &27.9 &65.0 &48.9 &59.0 &32.1 &46.8 &39.3 &44.3 \\
          &\textbf{TreePO} &54.0 &49.7 &28.4 &65.9 &49.5 &58.9 &32.0 &46.5 &39.4 &44.2 \\
          \cmidrule{2-12}
          &\cellcolor{gray!20}\textbf{TRACE} &\cellcolor{gray!20}55.7 &\cellcolor{gray!20}50.9 &\cellcolor{gray!20}29.5 &\cellcolor{gray!20}66.3 &\cellcolor{gray!20}\textbf{50.6} &\cellcolor{gray!20}61.2 &\cellcolor{gray!20}34.4 &\cellcolor{gray!20}48.8 &\cellcolor{gray!20}40.4 &\cellcolor{gray!20}\textbf{46.2} \\
      \midrule
      \multirow{5}{*}{\rotatebox{90}{\textbf{Qwen3-14B}}}
          &\textbf{ReAct} &46.4 &44.2 &22.8 &55.3 &42.2 &51.1 &23.6 &35.6 &22.8 &33.6 \\
          &\textbf{GRPO} &55.3 &52.7 &30.1 &66.8 &51.2 &70.6 &32.6 &40.9 &36.4 &46.1 \\
          &\textbf{PCL} &55.6 &52.9 &30.2 &67.2 &51.5 &70.0 &31.8 &41.4 &40.4 &45.9 \\
          &\textbf{TreePO} &55.1 &55.3 &31.0 &70.6 &53.0 &70.8 &33.0 &43.2 &39.4 &46.6 \\
          \cmidrule{2-12}
          &\cellcolor{gray!20}\textbf{TRACE} &\cellcolor{gray!20}57.8 &\cellcolor{gray!20}53.1 &\cellcolor{gray!20}33.7 &\cellcolor{gray!20}71.3 &\cellcolor{gray!20}\textbf{54.0} &\cellcolor{gray!20}72.4 &\cellcolor{gray!20}34.8 &\cellcolor{gray!20}44.6 &\cellcolor{gray!20}40.2 &\cellcolor{gray!20}\textbf{48.0} \\
      \bottomrule
  \end{tabular}}}
  \label{tab:multihop_bfcl_results}
\end{table*}

\begin{figure*}[t]
  \centering
  \begin{subfigure}[t]{0.32\textwidth}
    \centering
    \includegraphics[width=\linewidth]{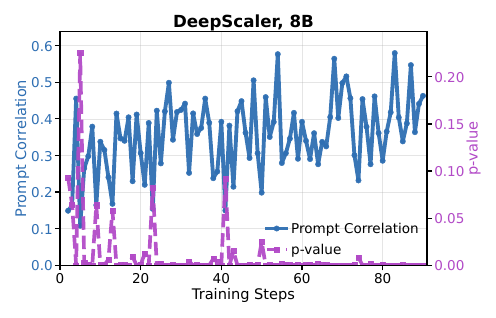}
  \end{subfigure}
  \begin{subfigure}[t]{0.32\textwidth}
    \centering
    \includegraphics[width=\linewidth]{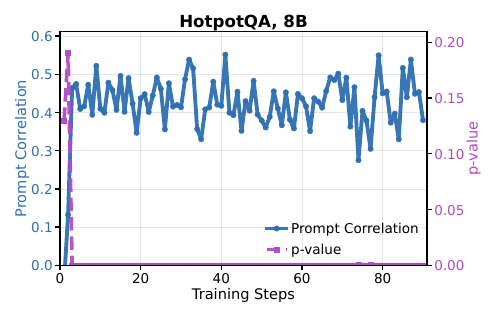}
  \end{subfigure}
  \begin{subfigure}[t]{0.32\textwidth}
    \centering
    \includegraphics[width=\linewidth]{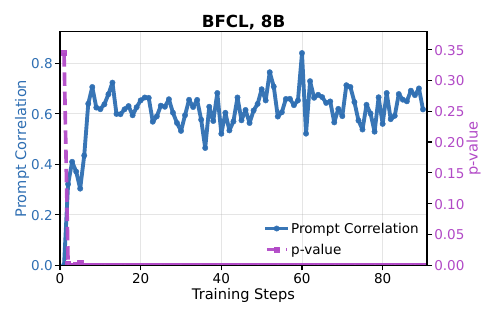}
  \end{subfigure}
  \vspace{+2pt}
  \begin{subfigure}[t]{0.32\textwidth}
    \centering
    \includegraphics[width=\linewidth]{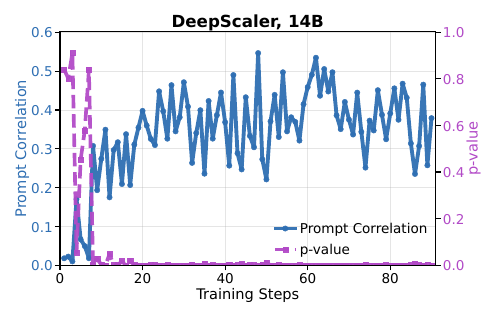}
  \end{subfigure}
  \begin{subfigure}[t]{0.32\textwidth}
    \centering
    \includegraphics[width=\linewidth]{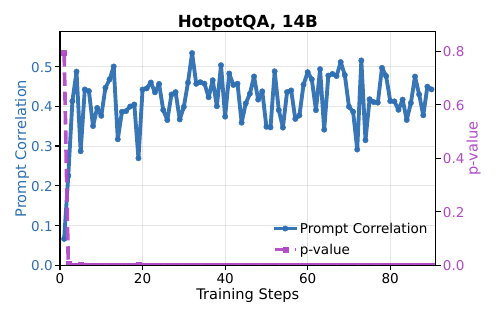}
  \end{subfigure}
  \begin{subfigure}[t]{0.32\textwidth}
    \centering
    \includegraphics[width=\linewidth]{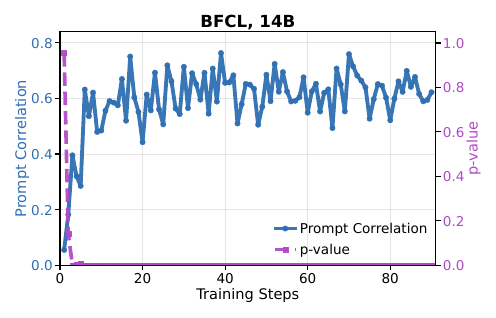}
  \end{subfigure}
  \caption{\textbf{Prompt-level prediction quality during training.} The six panels cover Mathematical Reasoning (DeepScaler), Multi-Hop QA (HotpotQA), and Function Calling (BFCL v4) with Qwen3-8B (top) and Qwen3-14B (bottom). The panels show Spearman's rank correlation between predicted prompt difficulty and empirical success rate, with the associated significance shown on the right axis.}
  \label{fig:app_prompt_spearman}
\end{figure*}

\begin{figure*}[t]
  \centering
  \begin{subfigure}[t]{0.32\textwidth}
    \centering
    \includegraphics[width=\linewidth]{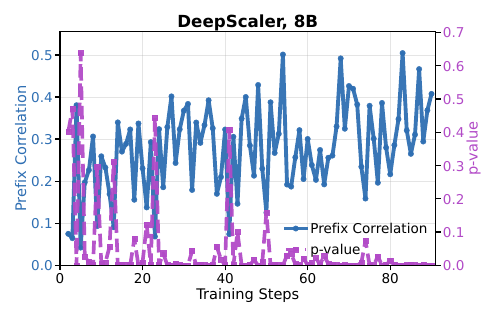}
  \end{subfigure}
  \begin{subfigure}[t]{0.32\textwidth}
    \centering
    \includegraphics[width=\linewidth]{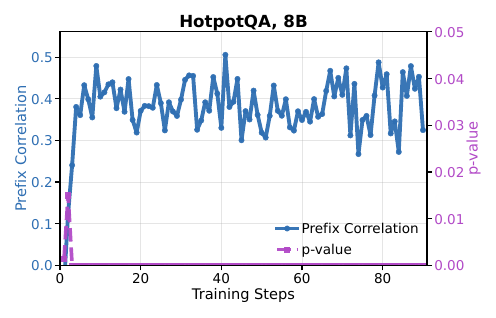}
  \end{subfigure}
  \begin{subfigure}[t]{0.32\textwidth}
    \centering
    \includegraphics[width=\linewidth]{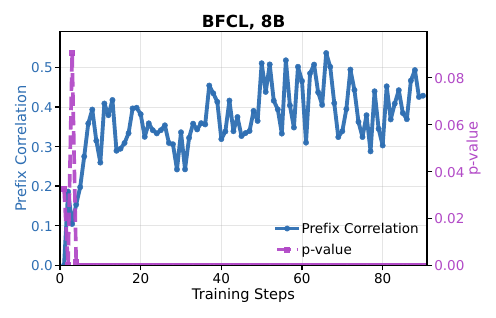}
  \end{subfigure}
  \vspace{+2pt}
  \begin{subfigure}[t]{0.32\textwidth}
    \centering
    \includegraphics[width=\linewidth]{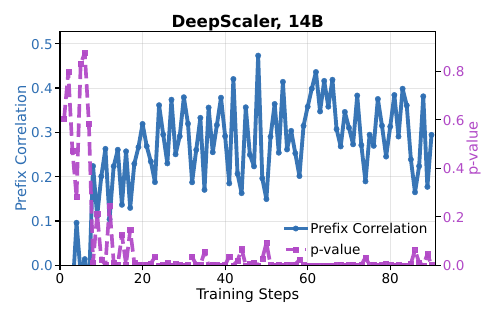}
  \end{subfigure}
  \begin{subfigure}[t]{0.32\textwidth}
    \centering
    \includegraphics[width=\linewidth]{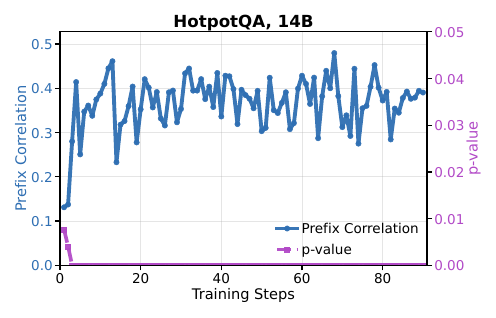}
  \end{subfigure}
  \begin{subfigure}[t]{0.32\textwidth}
    \centering
    \includegraphics[width=\linewidth]{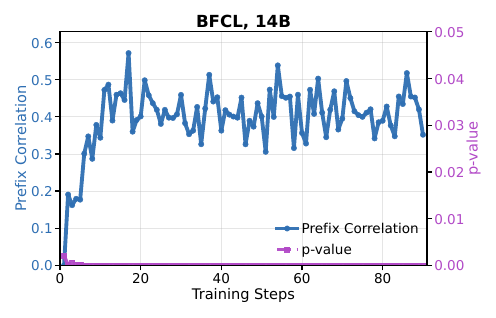}
  \end{subfigure}
  \caption{\textbf{Prefix-level prediction quality during training.} The six panels cover Mathematical Reasoning (DeepScaler), Multi-Hop QA (HotpotQA), and Function Calling (BFCL v4) with Qwen3-8B (top) and Qwen3-14B (bottom). The panels show Spearman's rank correlation between predicted prefix difficulty and empirical continuation success, with the associated significance shown on the right axis.}
  \label{fig:app_node_spearman}
\end{figure*}

\subsection{Root-to-Prefix Predictor Generalization}
\label{sec:root_prefix_predictor_generalization}

Using one shared predictor for roots and prefixes is not merely an implementation shortcut. A prefix history is an information-refined prompt: it contains the current plan, the executed actions, and the environment observations that shape what remains possible. Root-level supervision is denser because each root aggregates more terminal descendants and yields more reliable tree-backed targets, anchoring the predictor's global difficulty scale, while prefix-level supervision is sparser because fewer descendants sit below each prefix and the targets are noisier, but closer to the decisions that create local contrast. For training efficiency, predictor updates use all prompt-root examples together with only a small stream of prefix examples; nevertheless, the positive prefix-level correlations in Appendix~\ref{sec:additional_results} indicate that prefix difficulty is still ranked accurately enough for branching and that the predictor is not memorizing prompt identities. It learns a transferable history-conditioned notion of uncertainty, which is precisely the signal needed for targeted branching.

\begin{figure*}[t]
  \centering
  \begin{minipage}[t]{0.6\textwidth}
    \vspace*{0pt}
    \centering
    \captionsetup{aboveskip=0pt,belowskip=2pt}
    \captionof{table}{\textbf{Llama-3.2-3B Multi-Hop QA accuracy.} To test transfer beyond Qwen backbones, we report answer-match reward scores on four multi-hop QA benchmarks and their average.}
    \label{tab:llama3b_multihopqa_step30}
    \vspace{+10pt}
    \renewcommand\arraystretch{1.0}
    \resizebox{0.7\linewidth}{!}{
    \setlength{\tabcolsep}{1mm}{
    \begin{tabular}{cccccc}
        \toprule
        \textbf{Method} & \textbf{Hotpot} & \textbf{2Wiki} & \textbf{Musiq} & \textbf{Bamb} & \textbf{Avg} \\
        \midrule
        \textbf{ReAct} &6.7 &7.3 &2.9 &6.6 &5.9 \\
        \textbf{GRPO} &38.2 &30.5 &13.0 &32.3 &28.5 \\
        \textbf{PCL} &39.3 &32.5 &13.4 &30.5 &28.9 \\
        \textbf{TreePO} &39.7 &33.6 &13.6 &29.9 &29.2 \\
        \cmidrule{1-6}
        \cellcolor{gray!20}\textbf{TRACE} &\cellcolor{gray!20}41.0 &\cellcolor{gray!20}35.3 &\cellcolor{gray!20}15.2 &\cellcolor{gray!20}37.7 &\cellcolor{gray!20}\textbf{32.3} \\
        \bottomrule
    \end{tabular}}}
  \end{minipage}
  \hfill
  \begin{minipage}[t]{0.37\textwidth}
    \vspace*{0pt}
    \centering
    \captionsetup{aboveskip=0pt,belowskip=0pt}
    \includegraphics[width=\linewidth]{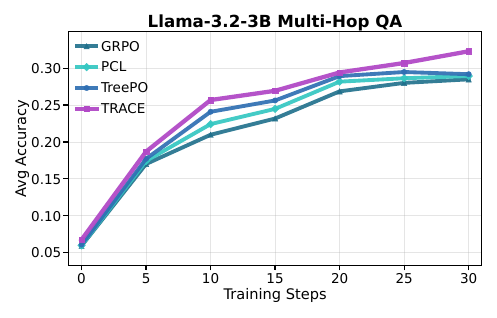}
    \vspace{-10pt}
    \caption{\textbf{Llama-3.2-3B Multi-Hop QA average accuracy.} We report the four-benchmark average over HotpotQA, 2WikiMultiHopQA, MuSiQue, and Bamboogle.}
    \label{fig:llama3b_multihopqa_accuracy}
  \end{minipage}
\end{figure*}

\begin{figure*}[t]
  \centering
  \begin{subfigure}[t]{0.32\textwidth}
    \centering
    \includegraphics[width=\linewidth]{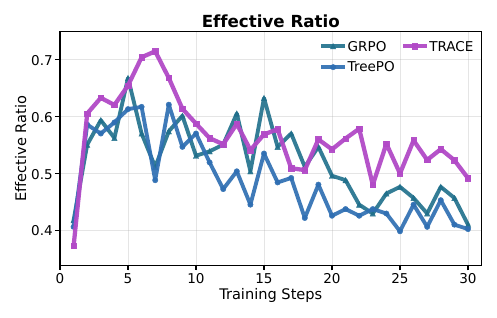}
  \end{subfigure}
  \begin{subfigure}[t]{0.32\textwidth}
    \centering
    \includegraphics[width=\linewidth]{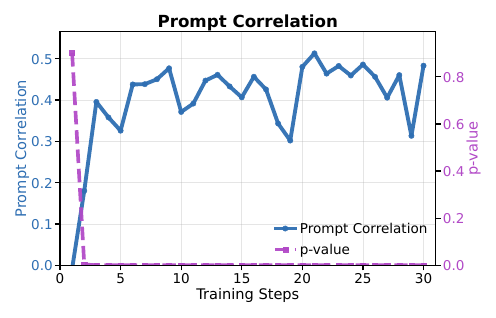}
  \end{subfigure}
  \begin{subfigure}[t]{0.32\textwidth}
    \centering
    \includegraphics[width=\linewidth]{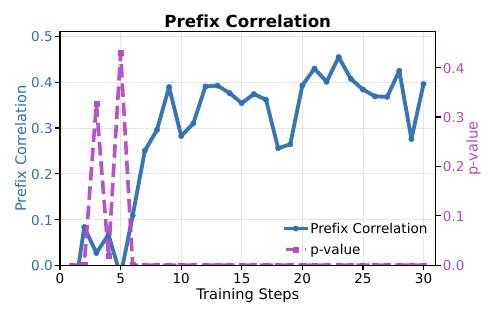}
  \end{subfigure}
  \caption{\textbf{Llama-3.2-3B allocation and predictor diagnostics on Multi-Hop QA (HotpotQA).} The panels follow Figures~\ref{fig:app_prompt_effective_ratio}--\ref{fig:app_node_spearman}, showing effective ratio, prompt-level correlation, and prefix-level correlation. Correlation panels show the predictor diagnostics.}
  \label{fig:llama3b_multihopqa_diagnostics}
\end{figure*}

\begin{figure*}[t]
  \centering
  \includegraphics[width=0.94\textwidth]{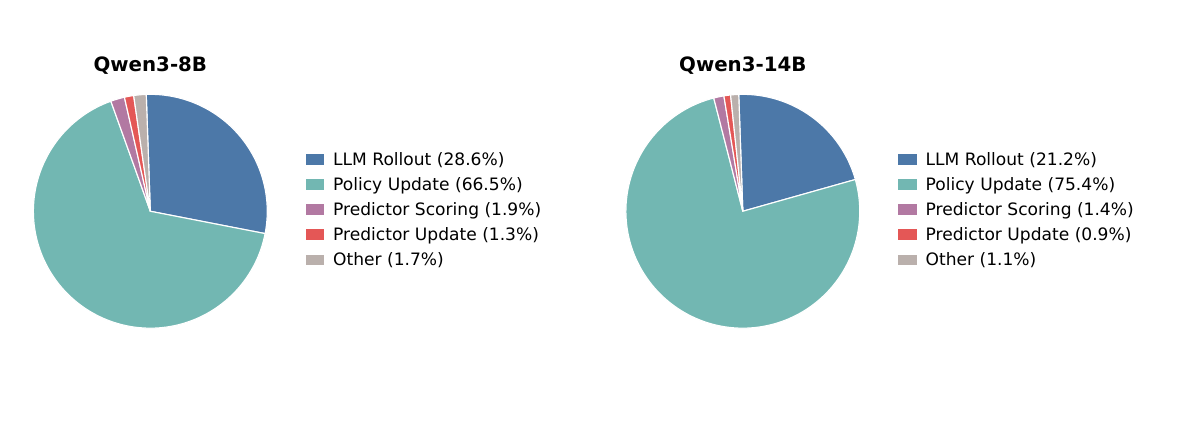}
  \vspace{-20pt}
  \caption{\textbf{Training time breakdown on Multi-Hop QA (HotpotQA).} We report average wall-clock time for TRACE with Qwen3-8B and Qwen3-14B. The current implementation uses an unoptimized predictor scoring path, while predictor parameter updates remain a small fraction of total runtime.}
  \label{fig:hotpotqa_time_breakdown}
\end{figure*}

\begin{figure*}[t]
  \centering
  \begin{subfigure}[t]{0.48\textwidth}
    \centering
    \includegraphics[width=\linewidth]{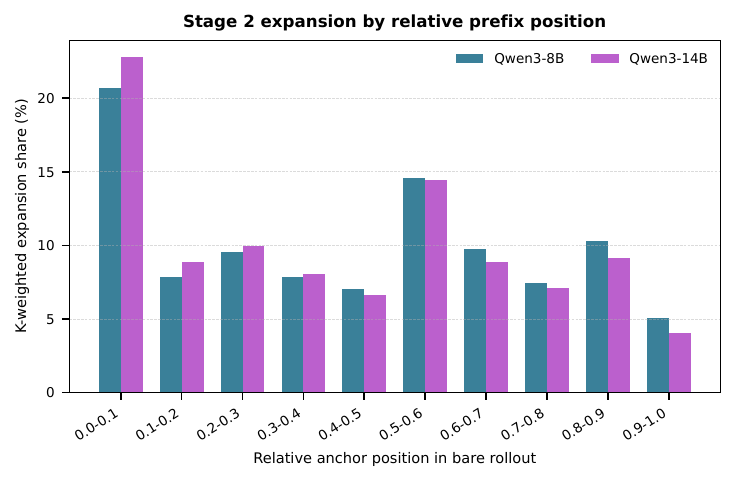}
    \caption{Relative prefix position}
  \end{subfigure}
  \hfill
  \begin{subfigure}[t]{0.48\textwidth}
    \centering
    \includegraphics[width=\linewidth]{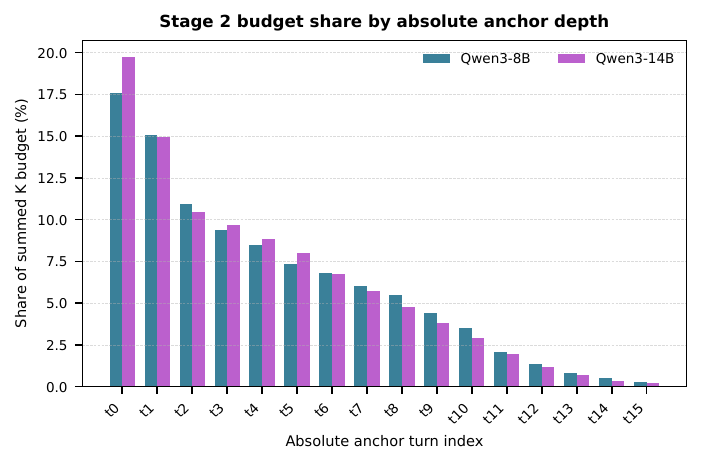}
    \caption{Absolute anchor depth}
  \end{subfigure}
  \caption{\textbf{Stage~2 allocation behavior on Function Calling (BFCL v4).} We average over the common late training window. The left panel bins anchors by relative position within their bare rollout and weights each anchor by the assigned continuation count $K_{i,j,t}$. The right panel reports the share of the total Stage~2 continuation budget assigned to each absolute anchor depth.}
  \label{fig:bfcl_stage2_allocation_behavior}
\end{figure*}

\begin{figure*}[t]
  \centering
  \begin{subfigure}[t]{0.48\textwidth}
    \centering
    \includegraphics[width=\linewidth]{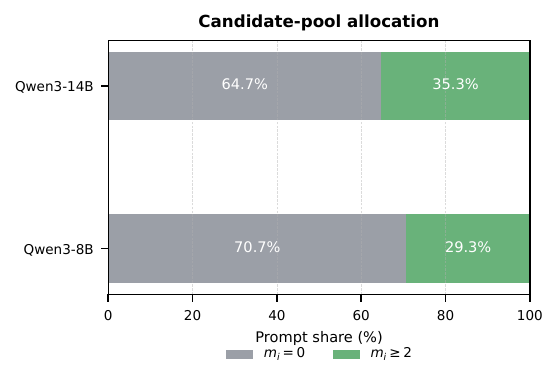}
    \caption{Candidate-pool allocation}
  \end{subfigure}
  \hfill
  \begin{subfigure}[t]{0.48\textwidth}
    \centering
    \includegraphics[width=\linewidth]{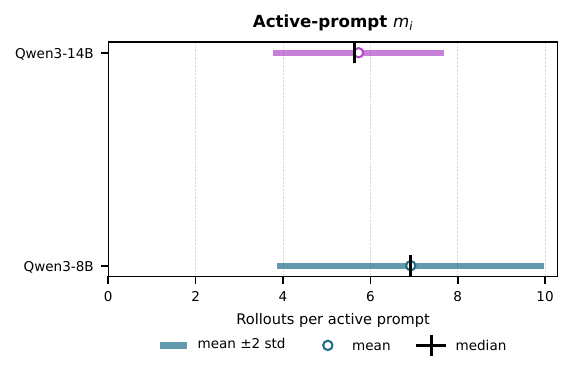}
    \caption{Active-prompt root counts}
  \end{subfigure}
  \caption{\textbf{Stage~1 root allocation summary on Function Calling (BFCL v4).} Panel (a) shows the fraction of candidate prompts receiving $m_i=0$ or $m_i\ge2$; by construction TRACE uses $m_i\in\{0\}\cup\{2,3,\ldots\}$, so $m_i=1$ has zero mass. Panel (b) shows active-prompt $m_i$ mean, median, and two-standard-deviation range over the same late training window. Exact per-integer counts within the active set require logging the full integer histogram of $m_i$.}
  \label{fig:bfcl_stage1_mi_summary}
\end{figure*}

\section{Extended Experimental Results}
\label{sec:additional_results}

\subsection{Detailed Benchmark Results}
\label{sec:app_detailed_results}

Figure~\ref{fig:hotpotqa_accuracy} shows training curves with respect to optimization steps, while Tables~\ref{tab:deepscaler_results} and~\ref{tab:multihop_bfcl_results} summarize the final endpoint performance. Overall, TRACE enables more effective RLVR training under the same rollout-budget setting. Compared with GRPO, TRACE improves the average scores by 0.7--2.8 points across Mathematical Reasoning, Multi-Hop QA, and Function Calling scenarios. Compared with random TreePO allocation, TRACE also consistently improves the reported averages, indicating that the gains come from learned allocation rather than from introducing a tree substrate alone. These improvements come with negligible predictor overhead, analyzed in Appendix~\ref{sec:app_computational_overhead}.

\subsection{Effective Ratio}

For the root or prefix stage $q\in\{\mathrm{root},\mathrm{pref}\}$, we use the anchor set $\mathcal{A}_q$, descendant set $\mathcal{D}_h(b_h)$, and activation event $\mathsf{Act}(h,b_h)$ from Eq.~\eqref{eq:local_gradient_contribution}. The effective ratio is the fraction of sampled terminal descendants attached to activated anchors:
\begin{equation}
  \hspace{-0.8em}
  \mathrm{Eff}_q(b)
  =
  \frac{
  \sum_{h\in\mathcal{A}_q}
  |\mathcal{D}_h(b_h)|\,
  \mathbb{I}\{\mathsf{Act}(h,b_h)\}}
  {\sum_{h\in\mathcal{A}_q}|\mathcal{D}_h(b_h)|}.
  \label{eq:effective_ratio_metric}
\end{equation}
Higher values indicate that more sampled trajectories can produce non-degenerate group-relative or local contrast. In our diagnostic figures and ablation tables, we report the prompt-level instantiation because it is shared by flat and tree rollouts. For a training batch $\mathcal{B}_s$, let $\mathcal{D}_i:=\mathcal{D}_{x_i}(b_{x_i})$ be the sampled terminal-descendant set below prompt root $x_i$ (Eq.~\eqref{eq:local_gradient_contribution}), and let $R_i=\{r(\mathcal{H}):\mathcal{H}\in\mathcal{D}_i\}$ denote their terminal rewards. We compute
\begin{equation}
  \mathrm{Eff}_{\mathrm{root}}(\mathcal{B}_s)
  =
  \frac{1}{|\mathcal{B}_s|}
  \sum_{x_i\in\mathcal{B}_s}
  \mathbb{I}
  \left[\mathrm{Var}(R_i)>0\right].
  \label{eq:prompt_effective_ratio_metric}
\end{equation}
Here, $\mathbb{I}[\cdot]$ is the indicator function. This metric measures the fraction of prompt-root groups that contain both successful and failed outcomes. It is an end-to-end batch informativeness measure: root allocation and useful prefix expansion can both increase the chance that a prompt produces mixed outcomes.

Figure~\ref{fig:app_prompt_effective_ratio} shows that TRACE consistently produces more informative prompt-root groups than GRPO and random TreePO allocation. This metric isolates whether the prompts entering an update yield non-degenerate outcome variation, so higher values indicate that learned allocation supplies denser group-relative signal under the same rollout budget.

\subsection{Spearman Correlation}

To evaluate the quality of difficulty prediction, we report Spearman's rank correlation coefficient~\citep{sedgwick2014spearman}. It measures the strength of the monotonic relationship between predicted and empirical difficulty rankings and is invariant to monotonic transformations, making it the right metric for allocation, where relative ordering matters more than calibration. For a root-level or prefix-level evaluation set $\mathcal{S}_q$, where $q\in\{\mathrm{root},\mathrm{pref}\}$, let $R_\psi(y)$ and $R_{\widehat V}(y)$ be the ranks of the predictor score $\tilde V_\psi(y)$ from Eq.~\eqref{eq:predictor_interface} and the tree-backed empirical target $\widehat V(y)$ from Eq.~\eqref{eq:node_value_recursive} within $\mathcal{S}_q$. We compute
\begin{equation}
  \begin{aligned}
  \rho_q
  &=
  \frac{
  \operatorname{cov}_{y\in\mathcal{S}_q}
  \!\left(R_\psi(y),R_{\widehat V}(y)\right)}
  {
  \operatorname{std}_{y\in\mathcal{S}_q}\!\left(R_\psi(y)\right)
  \operatorname{std}_{y\in\mathcal{S}_q}\!\left(R_{\widehat V}(y)\right)
  } .
  \end{aligned}
  \label{eq:spearman_metric}
\end{equation}
We also report the corresponding $p$-value under the null hypothesis that $\tilde V_\psi(y)$ and $\widehat V(y)$ are independent over $y\in\mathcal{S}_q$, which assesses whether the observed rank agreement is statistically significant.

Figures~\ref{fig:app_prompt_spearman} and~\ref{fig:app_node_spearman} show that the predictor learns stable difficulty rankings at both levels. Prompt-level correlation rises first because roots provide the densest supervision. Prefix-level correlation remains positive as well, even though predictor training is dominated by prompt-level examples and uses only limited prefix-level supervision to expose the history format. This transfer is consistent with the root-to-prefix generalization view in Appendix~\ref{sec:root_prefix_predictor_generalization}, suggesting that the predictor is not memorizing prompt IDs but naturally extending a root-trained difficulty model to serialized interaction histories, which is exactly the signal needed for targeted branching.

\subsection{Applicability to Other Model Families}

We further test TRACE on Llama-3.2-3B-Instruct to examine whether the allocation interface in Section~\ref{sec:method} transfers beyond the Qwen3 family. Consistent with the Multi-Hop QA trend in Section~\ref{sec:main_results}, Table~\ref{tab:llama3b_multihopqa_step30} shows that TRACE achieves the best average accuracy, improving over GRPO by 3.8 points and over random TreePO allocation by 3.1 points. The gains appear across HotpotQA, 2WikiMultiHopQA, MuSiQue, and Bamboogle, suggesting that the benefit is not tied to a single benchmark instance.

Figure~\ref{fig:llama3b_multihopqa_accuracy} shows that the advantage emerges during training rather than from a stronger starting point. Figure~\ref{fig:llama3b_multihopqa_diagnostics} provides the corresponding allocation diagnostics, showing that TRACE maintains a higher effective ratio and that the predictor preserves useful prompt- and prefix-level rankings. Together, these results support the central design claim that TRACE is a portable root-prefix allocation layer rather than a policy-backbone-specific heuristic.

\subsection{Computational Overhead}
\label{sec:app_computational_overhead}

As a representative case, we profile Multi-Hop QA training runs on HotpotQA to quantify where wall-clock time is spent. Figure~\ref{fig:hotpotqa_time_breakdown} decomposes the runtime into LLM rollout, policy update, predictor scoring, predictor update, and residual overhead. The decomposition uses nonoverlapping critical path timers in which predictor scoring is subtracted from trajectory collection and predictor update is separated from trajectory transformation. This avoids double counting nested diagnostic timers.

The dominant cost is still policy optimization, which takes 66.5\% of the Qwen3-8B runtime and 75.4\% of the Qwen3-14B runtime. The learnable predictor update is small, accounting for 1.3\% and 0.9\% of total runtime. Predictor scoring during root and prefix allocation accounts for 1.9\% and 1.4\% of total runtime. Because TRACE uses the same Qwen3-0.6B predictor for both policy scales, total predictor overhead decreases from 3.2\% on Qwen3-8B to 2.3\% on Qwen3-14B, while the policy update becomes increasingly dominant. These results indicate that TRACE trades a negligible allocation cost for more informative rollouts, and that further engineering improvements can primarily target batched prefix scoring.

\subsection{Allocation Behavior}
\label{sec:app_allocation_behavior}

We inspect Function Calling (BFCL v4) as a representative allocation trace because its longer function-calling episodes make turn-wise prefix allocation smoother than in shorter QA rollouts. Figure~\ref{fig:bfcl_stage2_allocation_behavior} shows two complementary views of Stage~2. The relative-position view normalizes each anchor by the length of its bare rollout, while the absolute-depth view reports the share of continuation budget assigned to each turn index. In Figure~\ref{fig:bfcl_stage2_allocation_behavior}(a), the budget-weighted prefix position remains near the middle of the bare rollout, matching the half-trajectory continuation approximation used in the budget accounting of Appendix~\ref{app:init_expand}. The absolute-depth view further shows that nontrivial budget remains across later prefixes up to depth 15, indicating that Stage~2 is not merely adding extra root-like samples.

Figure~\ref{fig:bfcl_stage1_mi_summary} summarizes Stage~1 root allocation over the candidate prompt pool. Most candidates receive $m_i=0$ and are skipped, while active prompts receive multiple bare rollouts, typically around five to seven in Function Calling (BFCL v4). The current logs record exact inactive/active ratios and active-set quantiles; per-integer counts within the active set require logging the full integer histogram of $m_i$.

\section{Theoretical Proof}
\label{sec:proof}

\subsection{Proof of Proposition~\ref{prop:prefix_predictability}}
\label{sec:prefix_predictability_proof}
\begin{proof}
Conditioned on $\mathcal{H}_t$, the $m$ terminal rewards in $\bar R_{t,m}$ are independent Bernoulli variables with mean $V_t^\pi$. Hence
\begin{equation}
  \mathbb{E}[\bar R_{t,m}\mid\mathcal{H}_t]
  =
  V_t^\pi,
\end{equation}
so the Bayes-optimal predictor of group correctness is $V_t^\pi$. Its optimal risk is
\begin{equation}
  \mathcal{E}_{t,m}^\star
  =
  \frac{1}{m}
  \mathbb{E}\big[V_t^\pi(1-V_t^\pi)\big].
\end{equation}
Because $V_t^\pi=\mathbb{E}[V_{t+1}^\pi\mid\mathcal{H}_t]$, the law of total variance gives
\begin{equation}
  \mathbb{E}\big[V_{t+1}^\pi(1-V_{t+1}^\pi)\big]
  \le
  \mathbb{E}\big[V_t^\pi(1-V_t^\pi)\big].
\end{equation}
Therefore,
\begin{equation}
  \hspace*{-0.4em}
  \mathcal{E}_{t,m}^\star-\mathcal{E}_{t+1,m}^\star
  =
  \frac{1}{m}
  \mathbb{E}\big[\operatorname{Var}(V_{t+1}^\pi\mid\mathcal{H}_t)\big]
  \ge 0.
\end{equation}
The inequality is strict exactly when the next prefix changes the conditional success probability with positive probability under $\mathcal{H}_t$. Summing the one-step inequalities gives $\mathcal{E}_{t,m}^\star\le\mathcal{E}_{0,m}^\star$ for every $t$.
\end{proof}

\subsection{Proof of Proposition~\ref{prop:energy_identity}}
\label{sec:energy_identity_proof}
\begin{proof}
Let $Z_t := V_t^\pi = \mathbb{E}_{\pi}\big[r(\mathcal{H}_T) \mid \mathcal{F}_t\big]$, where $\mathcal{F}_t:=\sigma(\mathcal{H}_t)$. Since $(Z_t)$ is a Doob martingale, we have $\mathbb{E}_{\pi}[Z_{s+1}\mid \mathcal{F}_s]=Z_s$ and therefore $\mathbb{E}_{\pi}[A_s^\pi\mid \mathcal{F}_s]=0$ for $A_s^\pi:=Z_{s+1}-Z_s$. Using
\begin{equation}
  Z_{s+1}^2 - Z_s^2 = (Z_{s+1}-Z_s)^2+2Z_s(Z_{s+1}-Z_s),
\end{equation}
and summing from $s=t$ to $T-1$, we obtain
\begin{equation}
  Z_T^2 - Z_t^2
  = \sum_{s=t}^{T-1}(A_s^\pi)^2 + 2\sum_{s=t}^{T-1}Z_sA_s^\pi.
\end{equation}
Taking conditional expectation with respect to $\mathcal{F}_t$, the second term vanishes by the tower property:
\begin{equation}
  \mathbb{E}_{\pi}[Z_sA_s^\pi\mid \mathcal{F}_t] =
  \mathbb{E}_{\pi}\big[
  \mathbb{E}_{\pi}[Z_sA_s^\pi\mid \mathcal{F}_s]
  \mid \mathcal{F}_t
  \big] =
  \mathbb{E}_{\pi}\big[
  Z_s\,\mathbb{E}_{\pi}[A_s^\pi\mid \mathcal{F}_s]
  \mid \mathcal{F}_t
  \big] = 0.
\end{equation}
Hence
\begin{equation}
  \mathbb{E}_{\pi}\left[[Z]_{t:T}\mid \mathcal{F}_t\right]
  = \mathbb{E}_{\pi}[Z_T^2\mid \mathcal{F}_t] - Z_t^2.
\end{equation}
Because the terminal reward is binary, $Z_T=r(\mathcal{H}_T)\in\{0,1\}$ and thus $Z_T^2=Z_T$. Therefore,
\begin{equation}
  \mathbb{E}_{\pi}[Z_T^2\mid \mathcal{F}_t]
  = \mathbb{E}_{\pi}[Z_T\mid \mathcal{F}_t]
  = Z_t,
\end{equation}
which yields
\begin{equation}
  \mathbb{E}_{\pi}\left[[Z]_{t:T}\mid \mathcal{F}_t\right]
  = Z_t - Z_t^2
  = V_t^\pi\bigl(1-V_t^\pi\bigr).
\end{equation}
This proves Eq.~\eqref{eq:energy_identity}.
\end{proof}

\subsection{Proof of Proposition~\ref{prop:trace_dominance}}
\label{app:allocation_dominance}

\begin{proof}
Fix an anchor $a$ and a budget decision $b$. Let $\mathsf{Act}(a,b)$ denote the event that the sampled descendants below $a$ contain at least one successful and one failed terminal history.

\noindent\textbf{Step 1: outcome contrast is the activation event.}
For a root group of size $m$, activation means that the group is non-degenerate:
\begin{equation}
  \Pr(\mathsf{Act}(a,m)\mid a)
  =
  1-p_a^m-(1-p_a)^m.
  \label{eq:app_root_activation}
\end{equation}
This is Eq.~\eqref{eq:root_esr_value} after replacing $p_a$ by the predictor. For a visited prefix occurrence with factual reward $r$, a new continuation matches the factual outcome with probability
\begin{equation}
  q_a(r)=rp_a+(1-r)(1-p_a).
  \label{eq:app_prefix_match}
\end{equation}
With $k$ new continuations, activation means at least one opposite-outcome sibling appears:
\begin{equation}
  \Pr(\mathsf{Act}(a,k)\mid a,r)
  =
  1-q_a(r)^k,
  \label{eq:app_prefix_activation}
\end{equation}
which is Eq.~\eqref{eq:pref_utility_practical}.

\noindent\textbf{Step 2: common optimizers have zero local gradient without activation.}
For pairwise Bradley--Terry or DPO-style losses, an informative pair consists of two terminal histories $\mathcal{H}^+,\mathcal{H}^-$ below $a$ with different rewards. Let
\[
  \ell_\theta(\mathcal{H}\mid a)
  =
  \log\frac{\pi_\theta(\mathcal{H}\mid a)}
           {\pi_{\mathrm{ref}}(\mathcal{H}\mid a)},
  \qquad
  \Delta_\theta
  =
  \ell_\theta(\mathcal{H}^+\mid a)
  -
  \ell_\theta(\mathcal{H}^-\mid a).
\]
Then the Bradley--Terry gradient is
\begin{equation}
  \nabla_\theta\log\sigma(\beta\Delta_\theta)
  =
  \beta\sigma(-\beta\Delta_\theta)
  \nabla_\theta\Delta_\theta
  =
  \beta\sigma(-\beta\Delta_\theta)
  \big[
  \nabla_\theta\log\pi_\theta(\mathcal{H}^+\mid a)
  - \nabla_\theta\log\pi_\theta(\mathcal{H}^-\mid a)
  \big]
  \label{eq:app_bt_gradient}
\end{equation}
If all sampled rewards below $a$ are identical, no such success--failure pair exists and the local pairwise contribution is zero.

For group-relative updates, let $\mathcal{H}_1,\ldots,\mathcal{H}_m$ be sampled descendants with rewards $R_s\in\{0,1\}$ and mean $\bar R$. Ignoring clipping and constants, the local group-relative gradient contribution is
\begin{equation}
  G_{\mathrm{grp}}(a)
  =
  \sum_{s=1}^{m}
  (R_s-\bar R)
  \nabla_\theta\log\pi_\theta(\mathcal{H}_s\mid a).
  \label{eq:app_group_grad}
\end{equation}
For any vectors $g_1,\ldots,g_m$, the algebraic identity
\begin{equation}
  \sum_{s=1}^{m}(R_s-\bar R)g_s
  =
  \frac{1}{m}\sum_{u<v}(R_u-R_v)(g_u-g_v),
  \label{eq:app_pair_decomp}
\end{equation}
holds. Taking $g_s=\nabla_\theta\log\pi_\theta(\mathcal{H}_s\mid a)$ shows that $G_{\mathrm{grp}}(a)$ is a sum of success--failure score differences; if the group is all-success or all-failure, every coefficient vanishes.

Now let $\widetilde G_a(b)$ denote the local gradient shape that the chosen optimizer would produce from the available success--failure contrasts below anchor $a$, and let $G_a(b)$ be the realized local gradient contribution. The pairwise and group-relative arguments above show that no such contribution is available when no opposite-outcome descendants are observed. Tree-aware optimizers only redistribute such leaf-level contrast through the tree, so they also have zero local contrast contribution below $a$ when $\mathsf{Act}(a,b)$ fails. Hence, for all optimizer classes considered here,
\begin{equation}
  G_a(b)
  =
  \mathbb{I}\{\mathsf{Act}(a,b)\}\widetilde G_a(b).
  \label{eq:app_gradient_activation}
\end{equation}

\noindent\textbf{Step 3: activation factors the expected squared gradient norm.}
Taking squared norms and conditional expectation in Eq.~\eqref{eq:app_gradient_activation} gives
\begin{equation}
  \mathbb{E}\!\left[\|G_a(b)\|^2\mid a\right]
  =
  \Pr(\mathsf{Act}(a,b)\mid a) \cdot
  \mathbb{E}\!\left[
  \|\widetilde G_a(b)\|^2
  \mid \mathsf{Act}(a,b),a
  \right].
  \label{eq:app_gradient_norm_factorization}
\end{equation}
By definition, $V_a(b)$ denotes the appropriate activation probability: $\Pr(\mathsf{Act}(a,b)\mid a)$ for roots and $\Pr(\mathsf{Act}(a,b)\mid a,r)$ for prefix anchors with factual reward $r$. If the conditional squared gradient norm lies in $[c_-,c_+]$ across candidate anchors, then the expected squared local gradient norm is bounded between $c_-V_a(b)$ and $c_+V_a(b)$.

\noindent\textbf{Step 4: dominance of local gradient energy.}
For a stage $q\in\{\mathrm{root},\mathrm{pref}\}$, let $h$ index the candidate anchors in that stage and let $b_h$ be the budget assigned to anchor $h$. Step 3 gives the total expected squared local gradient norm in the form
\[
  \mathcal{E}_q(b)
  =
  \sum_h V_h(b_h) C_h(b_h),
\]
where $C_h(b_h)$ is the conditional squared gradient scale after activation. This factor is optimizer- and anchor-dependent, while $V_h(b_h)$ is exactly the activation probability controlled by rollout allocation. Under the normalized-gradient condition in Proposition~\ref{prop:trace_dominance}, $C_h(b_h)=c_q$ for every anchor in stage $q$. TRACE therefore optimizes the gradient-energy component
\begin{equation}
  \mathcal{J}_q(b)
  :=
  \sum_{h}V_h(b_h),
  \label{eq:app_generic_allocation}
\end{equation}
where $V_h$ is Eq.~\eqref{eq:root_esr_value} for $q=\mathrm{root}$ and Eq.~\eqref{eq:pref_utility_practical} for $q=\mathrm{pref}$. TRACE solves the resulting combinatorial optimization over the feasible set $\mathcal{F}_q$. If $b_q^\star\in\arg\max_{b\in\mathcal{F}_q}\mathcal{J}_q(b)$ and $b_q^{u}\in\mathcal{F}_q$ is a uniform baseline, then
\begin{equation}
  \begin{gathered}
  \mathcal{E}_q(b_q^\star)
  =
  c_q\mathcal{J}_q(b_q^\star)
  \ge
  c_q\mathcal{J}_q(b_q^u)
  =
  \mathcal{E}_q(b_q^u).
  \end{gathered}
\end{equation}
This proves the first line of Eq.~\eqref{eq:trace_dominates_uniform} for both $q=\mathrm{root}$ and $q=\mathrm{pref}$.

For the combined tree update, write $G_r:=G_{\mathrm{root}}(b_{\mathrm{root}})$ and $G_p:=G_{\mathrm{pref}}(b_{\mathrm{pref}})$. Expanding the squared norm gives
\[
  \mathbb{E}\|G_r+G_p\|^2
  =
  \mathcal{E}_{\mathrm{root}}(b_{\mathrm{root}})
  +
  \mathcal{E}_{\mathrm{pref}}(b_{\mathrm{pref}})
  +
  2\mathbb{E}\langle G_r,G_p\rangle.
\]
The two energy terms are no smaller under TRACE by the stage-wise result above. The cross-term condition in Proposition~\ref{prop:trace_dominance} states that TRACE does not reduce $\mathbb{E}\langle G_r,G_p\rangle$ relative to uniform, so the full expected squared gradient norm is also no smaller. Strictness holds for any stage whose uniform allocation is not an optimizer of $\mathcal{J}_q$.
This proves Proposition~\ref{prop:trace_dominance}.
\end{proof}

\section{Experimental Details}
\label{sec:experimental_details}

\subsection{Tasks}

\subsubsection{Mathematical Reasoning}

\paragraph{Training dataset.}
We train on the DeepScaler~\citep{luo2025deepscaler} corpus, a collection of 40,315 competition-level mathematics problems. We use the public DeepScaleR preview dataset hosted at \url{https://huggingface.co/datasets/agentica-org/DeepScaleR-Preview-Dataset}. Each rollout is executed by a tool-augmented agent with access to a Python interpreter for calculation and verification.

\paragraph{Evaluation benchmarks.}
We evaluate in-distribution mathematical reasoning on AIME24, AMC23, MATH500~\citep{lightman2024let}, MinervaMath~\citep{lewkowycz2022solving}, and OlympiadBench~\citep{he2024olympiadbench}, using the math evaluation sets hosted at \url{https://huggingface.co/datasets/math-ai}. Per-dataset scores use the same sampling multiplicities as our evaluation scripts: AIME24 is reported as avg@32, AMC23 as avg@16, and MATH500, MinervaMath, and OlympiadBench as avg@4. To test out-of-distribution generalization, we additionally evaluate on MMLU-Pro~\citep{wang2024mmlu}, ARC-Challenge~\citep{clark2018think}, and GPQA-diamond~\citep{rein2023gpqa}, reported as avg@1, avg@4, and avg@8 respectively. The ``Avg'' columns in the tables are arithmetic averages across the corresponding benchmark columns.

\paragraph{Reward function.}
We use a binary verifier reward: a trajectory receives reward $1$ if the final extracted answer matches the ground truth and $0$ otherwise. The same reward interface is used for base-model evaluation and RL checkpoints.

\subsubsection{Multi-Hop QA}

\paragraph{Training dataset.}
We train on the training split of HotpotQA~\citep{yang2018hotpotqa}, using the HuggingFace version hosted at \url{https://huggingface.co/datasets/hotpotqa/hotpot_qa}. The agent follows a ReAct-style search-and-answer loop, issuing search queries to a local E5 retrieval server~\citep{wang2022text} built over a Wikipedia dump~\citep{karpukhin2020dense}. This keeps retrieval deterministic and avoids external search-engine drift during RL training.

\paragraph{Evaluation benchmarks.}
We evaluate on the HotpotQA validation split, 2WikiMultiHopQA~\citep{ho2020constructing}, MuSiQue~\citep{trivedi2022musique}, and Bamboogle~\citep{press2023measuring}. The additional held-out sets are prepared from the FlashRAG benchmark repository at \url{https://huggingface.co/datasets/RUC-NLPIR/FlashRAG_datasets}, with canonical sources at \url{https://github.com/Alab-NII/2wikimultihop}, \url{https://github.com/StonyBrookNLP/musique}, and \url{https://huggingface.co/datasets/chiayewken/bamboogle}. These datasets stress different forms of multi-hop evidence aggregation, from entity chaining to compositional question decomposition. Multi-Hop QA evaluation uses one sampled trajectory per question, so each per-dataset score is avg@1. The reported ``Avg'' is the arithmetic average over the four QA benchmarks when per-checkpoint evaluation is available.

\paragraph{Reward function.}
We use the implemented HotpotQA-style search reward as both the training reward and evaluation score. The evaluator first extracts a concise final answer and normalizes it following SQuAD/HotpotQA conventions. An exact match receives reward $1$; otherwise, the evaluator computes token-level F1 against all reference answers, treats predictions with F1 at least $0.3$ as partially correct, and assigns reward equal to the F1 score. Predictions below this threshold receive reward $0$.

\subsubsection{Function Calling}

\paragraph{Dataset.}
We use the multi-turn split of BFCL v4~\citep{patil2025berkeley}, rather than the full BFCL v4 suite, which also contains many single-turn, live, memory, and web-search categories. We use the four agentic multi-turn subsets from \url{https://huggingface.co/datasets/gorilla-llm/Berkeley-Function-Calling-Leaderboard}: base, long-context, missing-function, and missing-parameter. We split these multi-turn examples into an 80\% training portion and a 20\% held-out test portion. The agent observes API specifications, calls functions over multiple turns, and receives environment feedback after each call.

\paragraph{Evaluation protocol.}
We follow the official BFCL evaluation protocol and report success rates over the base, missing-function, missing-parameter, and long-context subsets. Validation uses four sampled trajectories per example, so subset scores are reported as avg@4. The reported ``Avg'' is the arithmetic average over the four multi-turn BFCL subsets. A trajectory is correct only when the executed API sequence and final task outcome satisfy the evaluator.

\subsection{Models}

All main experiments use Qwen3-8B and Qwen3-14B~\citep{yang2025qwen3} policy backbones, loaded from the official Qwen model family. We additionally run Multi-Hop QA experiments with Llama-3.2-3B-Instruct~\citep{grattafiori2024llama} as an extra backbone in Appendix~\ref{sec:additional_results}. The ReAct baseline evaluates the corresponding base or instruction-tuned model in the same multi-turn agent scaffold without RL fine-tuning. TRACE uses a lightweight Qwen3-0.6B~\citep{yang2025qwen3} model as the prefix predictor. The relevant model repositories are:
\begin{itemize}[leftmargin=10pt]
    \item Qwen3-8B: \url{https://huggingface.co/Qwen/Qwen3-8B};
    \item Qwen3-14B: \url{https://huggingface.co/Qwen/Qwen3-14B};
    \item Llama-3.2-3B-Instruct: \url{https://huggingface.co/meta-llama/Llama-3.2-3B-Instruct};
    \item Qwen3-0.6B predictor backbone: \url{https://huggingface.co/Qwen/Qwen3-0.6B}.
\end{itemize}

\subsection{Training Details}

All RL experiments are implemented in the rLLM/verl~\citep{sheng2024hybridflow,rllm2025} stack with asynchronous vLLM rollout workers. Unless otherwise stated, actor optimization uses Adam with learning rate $1\mathrm{e}{-6}$, no KL penalty, sequence-mean token aggregation, gradient clipping at $1.0$, and Clip-Higher ranges $\epsilon_{\mathrm{low}}=0.2$ and $\epsilon_{\mathrm{high}}=0.28$. We use FSDP with parameter and optimizer offloading for the actor and reference model. Main training runs use two nodes with eight GPUs per node.

The environment horizon and sequence lengths are task-specific. For Multi-Hop QA, we use maximum prompt length 2048, maximum response length 6144, and at most 5 agent turns. For Mathematical Reasoning (DeepScaler), we use maximum prompt length 2048, maximum response length 14336, and at most 5 agent turns. For Function Calling (BFCL v4), we use maximum prompt length 12800, maximum response length 8192, and at most 40 agent turns to accommodate long multi-turn tool-use traces.

Rollout sampling follows each benchmark's baseline setting. Multi-Hop QA uses temperature $1.0$ and top-$p=1.0$ during training. Mathematical Reasoning (DeepScaler) uses temperature $0.6$ and top-$p=0.95$, matching the math evaluation style. Function Calling (BFCL v4) uses temperature $0.9$ and top-$p=1.0$.

\subsection{Policy Optimizers and TreePO Instances}
\label{app:treegrpo_instance}

TRACE is optimizer-agnostic and outputs tree-structured rollout data that can be consumed by different tree-aware RL objectives. In the flat baseline, GRPO normalizes terminal rewards within a response group and removes the value network used in PPO. Let $\mathcal{B}_{\mathrm{G}}$ denote $x\sim\mathcal{D}$ and $\{y_i\}_{i=1}^{k}\sim\pi_{\theta_{\mathrm{old}}}(\cdot\mid x)$, and write $\bar\rho=\mathrm{clip}(\rho,1-\epsilon_{\mathrm{low}},1+\epsilon_{\mathrm{high}})$. GRPO optimizes\par
\begin{equation}
  \mathcal{J}_{\mathrm{GRPO}}(\theta)=
  \mathbb{E}_{\mathcal{B}_{\mathrm{G}}}
  \Bigg[\frac{1}{k}\sum_{i=1}^{k}\frac{1}{|y_i|}\sum_{t=1}^{|y_i|}\bigg(
  \min\Big(\rho_{i,t}\hat A_i,\bar\rho_{i,t}\hat A_i\Big)
  -\beta D_{\mathrm{KL}}(\pi_{\theta}\Vert\pi_{\mathrm{ref}})\bigg)\Bigg],
  \label{eq:app_grpo_objective}
\end{equation}
where $\rho_{i,t}=\frac{\pi_{\theta}(y_{i,t}\mid x,y_{i,<t})}{\pi_{\theta_{\mathrm{old}}}(y_{i,t}\mid x,y_{i,<t})}$ and
\begin{equation}
  \hat A_i =
  \frac{r_i-\mathrm{mean}(\{r_j\}_{j=1}^{k})}
  {\mathrm{std}(\{r_j\}_{j=1}^{k})}.
  \label{eq:app_grpo_advantage}
\end{equation}

For tree-structured training, we use two concrete optimizers on the same rollout-tree batch $\mathcal{B}_{T}$, which samples $x\sim\mathcal{D}$ and a rollout tree $\mathcal{G}(x)$ from $\pi_{\theta_{\mathrm{old}}}$ under the initialize-then-expand construction in \S\ref{sec:trace_tree_construction}. For stable training, Mathematical Reasoning (DeepScaler) and Multi-Hop QA (HotpotQA) use TreeRPO~\citep{yang2025treerpo}, while Function Calling (BFCL v4) uses Tree-GRPO~\citep{ji2025tree}. For any tree node $y$, let $\mathcal{C}(y)$ denote its child set as in Eq.~\eqref{eq:node_value_recursive} and let $R(y)$ be the scalar reward backed up by the chosen optimizer; at a terminal leaf with history $\mathcal{H}_T$, we set $R(y)=r(\mathcal{H}_T)$ from Eq.~\eqref{eq:rlvr_objective}. In TreeRPO, each internal node $y$ backs up the average child reward
\begin{equation}
  R(y)=\frac{1}{|\mathcal{C}(y)|}\sum_{c\in\mathcal{C}(y)}R(c).
  \label{eq:app_treerpo_backup}
\end{equation}
For every parent $p$, its child set $\mathcal{C}(p)$ forms a local group, with child advantage
\begin{equation}
  \begin{gathered}
  \hat A_c^{\mathrm{TR}}=
  \frac{R(c)-\mu_p}{\sigma_p},\\
  \mu_p=\frac{1}{|\mathcal{C}(p)|}\sum_{c'\in\mathcal{C}(p)}R(c').
  \end{gathered}
  \label{eq:app_treerpo_advantage}
\end{equation}
Let $\mathcal{E}_{T}$ denote the set of parent-child edges $(p,c)$ in $\mathcal{G}(x)$ used for prefix-level sibling comparisons, and let $o_c$ be the token sequence generated on child branch $c$. TreeRPO applies the same clipped GRPO-style update as Eq.~\eqref{eq:app_grpo_objective}:\par
\begin{equation}
  \mathcal{J}_{\mathrm{TR}}(\theta)=
  \mathbb{E}_{\mathcal{B}_{T}}\Bigg[
  \frac{1}{|\mathcal{E}_{T}|}\sum_{(p,c)\in\mathcal{E}_{T}}
  \frac{1}{|o_c|}\sum_{t=1}^{|o_c|}\bigg(
  \min\Big(\rho_{c,t}\hat A_c^{\mathrm{TR}},
  \bar\rho_{c,t}\hat A_c^{\mathrm{TR}}\Big)
  -\beta D_{\mathrm{KL}}(\pi_{\theta}\Vert\pi_{\mathrm{ref}})\bigg)
  \Bigg],
  \label{eq:app_treerpo_objective}
\end{equation}
where $\rho_{c,t}=\frac{\pi_{\theta}(o_{c,t}\mid p,o_{c,<t})}{\pi_{\theta_{\mathrm{old}}}(o_{c,t}\mid p,o_{c,<t})}$.

Tree-GRPO normalizes terminal outcome rewards within an intra-tree group (sibling branches sharing the same parent prefix) and an inter-tree group (all complete trajectories for the same prompt). Let $\{y_i\}_{i=1}^{G}$ denote the $G$ complete trajectories produced by tree search for prompt $x$, with terminal rewards $R(y_i)=r(\mathcal{H}_{T}^{(i)})$. Their combined advantage is
\begin{equation}
  \begin{gathered}
  \hat A_i^{\mathrm{intra/inter}}=
  \frac{R(y_i)-\mu_{\mathrm{intra/inter}}}{\sigma_{\mathrm{intra/inter}}},\\
  \hat A_i^{\mathrm{TG}}=\hat A_i^{\mathrm{intra}}+\hat A_i^{\mathrm{inter}}.
  \end{gathered}
  \label{eq:app_treegrpo_advantage}
\end{equation}
Tree-GRPO then applies the same clipped token-level update as GRPO:\par
\begin{equation}
  \mathcal{J}_{\mathrm{TG}}(\theta)
  =
  \mathbb{E}_{\mathcal{B}_{T}}
  \Bigg[\frac{1}{G}\sum_{i=1}^{G}\frac{1}{|y_i|}\sum_{t=1}^{|y_i|}\bigg(
  \min\Big(\rho_{i,t}\hat A_i^{\mathrm{TG}},
  \bar\rho_{i,t}\hat A_i^{\mathrm{TG}}\Big)
  -\beta D_{\mathrm{KL}}(\pi_{\theta}\Vert\pi_{\mathrm{ref}})\bigg)\Bigg],
  \label{eq:app_treegrpo_objective}
\end{equation}
where $\rho_{i,t}=\frac{\pi_{\theta}(y_{i,t}\mid x,y_{i,<t})}{\pi_{\theta_{\mathrm{old}}}(y_{i,t}\mid x,y_{i,<t})}$.
In all cases, TreePO denotes the same tree-aware optimizer family with random root and prefix allocation, while TRACE changes only the allocation policy that decides where tree branches are collected.

\subsection{Rollout Construction and Budget Accounting}
\label{app:init_expand}

All methods draw from the same task-dependent candidate pool, but they consume it differently. For the fixed-size prompt baselines, Multi-Hop QA selects 256 prompts from 512 candidates, while Mathematical Reasoning (DeepScaler) and Function Calling (BFCL v4) select 128 prompts from 256 candidates. GRPO selects this subset uniformly, and PCL selects it with its difficulty predictor. Each selected prompt then receives 8 flat rollouts. TreePO uses the same selected-prompt count and total budget, but collects a random tree by sampling 4 bare rollouts per selected prompt and randomly choosing two branch points in each bare rollout.

TRACE does not first commit to a fixed number of selected prompts. It solves Stage~1 over the full candidate pool and assigns integer root counts $m_i$, where $m_i=0$ skips a candidate and $m_i>0$ determines how many bare rollouts are generated. Figure~\ref{fig:bfcl_stage1_mi_summary} shows that the resulting active-prompt fraction does not exceed the fixed-size prompt baselines, so TRACE does not gain an advantage from exposure to more distinct prompts. In Multi-Hop QA we use $M=1024$ root rollout budget units and $N=2$ continuation budget units per active root; in Mathematical Reasoning (DeepScaler) and Function Calling (BFCL v4) we use $M=512$ and $N=2$. After the $m_i$ bare rollouts in Eq.~\eqref{eq:bare_rollout_generation} finish, each trajectory induces visited prefix histories $\mathcal{H}_{i,j,t}$ for $1\le t<T_{i,j}$. Stage~2 allocates $m_iN$ continuation slots over these prefix occurrences and samples the additional suffixes in Eq.~\eqref{eq:prefix_branch_generation}. We use one expansion round in the main experiments, so newly generated suffixes are evaluated as terminal branches and are not recursively expanded.

For budget accounting, we count a root rollout as one trajectory unit and approximate a continuation from an internal prefix as one half trajectory in expectation~\citep{ji2025tree,hou2025treerl}. Thus, ignoring rare fallback cases, TRACE uses approximately $M(1+N/2)$ trajectory units and is matched to a flat baseline with $P$ selected prompts and $G$ rollouts per prompt by choosing $(M,N)$ such that $M(1+N/2)\approx PG$. If no valid prefix candidate is available for a prompt, we convert the local continuation budget into additional root rollouts. In implementation, Stage~2 can be triggered per prompt as soon as that prompt's bare rollouts finish, which avoids a batch-wide straggler barrier while preserving the logical two-stage semantics.

\subsection{Prefix Predictor Supervision and Implementation}
\label{app:value_impl}

The prefix predictor $\tilde V_\psi$ is implemented as a Qwen3-0.6B critic model, extending the lightweight value-model setup used by PCL~\citep{gao2025prompt,qu2026small} from prompt roots to serialized prefix histories. For each queried anchor, the input is the serialized prompt and the interaction history up to that prompt root or prefix. The predictor outputs a scalar success score in $[0,1]$, which is used by both Stage~1 root allocation and Stage~2 prefix allocation. For multiple prefixes under the same prompt, prefix-cache reuse amortizes the shared prompt and early-history computation, keeping the additional scoring overhead small.

Following Eqs.~\eqref{eq:node_value_recursive}--\eqref{eq:value_loss}, the predictor is updated online after each rollout step by fitting scalar tree-backed success targets with an MSE value-model objective. Each update uses the current step's prompt-root examples together with a small stream of prefix examples, with prefix samples accounting for 6\% of the predictor batch. Predictor learning rates are $3\mathrm{e}{-5}$ for Multi-Hop QA and Function Calling (BFCL v4), and $1.5\mathrm{e}{-5}$ for Mathematical Reasoning (DeepScaler). We train for two predictor epochs per rollout on Multi-Hop QA and Function Calling (BFCL v4), and four epochs per rollout on Mathematical Reasoning (DeepScaler). Because tree-backed targets are refreshed as the policy improves, the predictor learns under a non-stationary target distribution and thus faces the stability--plasticity dilemma between adapting to new targets and mitigating catastrophic forgetting. As future work, continual-learning techniques could better manage this dilemma~\citep{kirkpatrick2017overcoming,li2017learning,zou2025structural,zou2026flycl,li2026enhancing}, and parameter-efficient fine-tuning~\citep{hu2022lora,zou2025flylora} could reduce update overhead without sacrificing ranking quality.

\section{Data Examples}
\label{appsec:dataexample}

This section shows representative examples from the actual rLLM data and agent interfaces used by our training scripts.  We include the system message seen by the agent, including the tool schemas injected by \texttt{ToolAgent}.

\begingroup
\let\tiny\footnotesize
\let\scriptsize\footnotesize

\begin{tcolorbox}[example, title=Mathematical Reasoning (DeepScaler) Example]
\textbf{System prompt:}

{\tiny\ttfamily\raggedright You are a math assistant that can write python to solve math problems.\par \par \# Tools\par \par You may call one or more functions to assist with the user query.\par \par You are provided with function signatures within \textless{}tools\textgreater{}\textless{}/tools\textgreater{} XML tags:\par \textless{}tools\textgreater{}\par [\par   \{\par     "type": "function",\par     "function": \{\par       "name": "python",\par       "description": "Execute Python code in a sandboxed environment. Returns results and standard output/error.",\par       "parameters": \{\par         "type": "object",\par         "properties": \{\par           "code": \{\par             "type": "string",\par             "description": "Execute Python code in a sandboxed environment. Returns results and standard output/error."\par           \},\par           "timeout": \{\par             "type": "integer",\par             "description": "Maximum execution time in seconds before timing out",\par             "default": 12\par           \}\par         \},\par         "required": [\par           "code"\par         ]\par       \}\par     \}\par   \}\par ]\par \textless{}/tools\textgreater{}\par \par For each function call, return a json object with function name and arguments within \textless{}tool\_call\textgreater{}\textless{}/tool\_call\textgreater{} XML tags:\par \textless{}tool\_call\textgreater{}\par \{"name": \textless{}function-name\textgreater{}, "arguments": \textless{}args-json-object\textgreater{}\}\par \textless{}/tool\_call\textgreater{}\par}

\medskip
\textbf{Prompt:}

{\scriptsize\ttfamily\raggedright The operation \$\textbackslash{}otimes\$ is defined for all nonzero numbers by \$a \textbackslash{}otimes b = \textbackslash{}frac\{a\textasciicircum{}\{2\}\}\{b\}\$. \par Determine \$[(1 \textbackslash{}otimes 2) \textbackslash{}otimes 3] - [1 \textbackslash{}otimes (2 \textbackslash{}otimes 3)]\$.\par}

\medskip
\textbf{Ground-Truth Answer:}

{\scriptsize\ttfamily\raggedright -\textbackslash{}frac\{2\}\{3\}\par}
\end{tcolorbox}

\begin{tcolorbox}[example, title=Multi-Hop QA (HotpotQA) Example]
\textbf{System prompt:}

{\tiny\ttfamily\raggedright You are a helpful AI assistant that can search for information to answer questions accurately.\par \par When answering questions:\par 1. Use the available search tools to find relevant and reliable information\par 2. Synthesize information from multiple sources when needed\par 3. Provide accurate and comprehensive answers based on your search results\par 4. Always put your final answer in \textbackslash{}boxed\{\} format\par \par For example:\par - If the answer is "American", write: \textbackslash{}boxed\{American\}\par - If the answer is "yes", write: \textbackslash{}boxed\{yes\}\par - If the answer is a year like "1985", write: \textbackslash{}boxed\{1985\}\par \par Remember to search thoroughly and provide your final answer clearly within the \textbackslash{}boxed\{\} format.\par \par \# Tools\par \par You may call one or more functions to assist with the user query.\par \par You are provided with function signatures within \textless{}tools\textgreater{}\textless{}/tools\textgreater{} XML tags:\par \textless{}tools\textgreater{}\par [\par   \{\par     "type": "function",\par     "function": \{\par       "name": "local\_search",\par       "description": "Search for information using a dense retrieval server with Wikipedia corpus",\par       "parameters": \{\par         "type": "object",\par         "properties": \{\par           "query": \{\par             "type": "string",\par             "description": "Search query to retrieve relevant documents"\par           \},\par           "top\_k": \{\par             "type": "integer",\par             "description": "Number of results to return (default: 10)",\par             "minimum": 1,\par             "maximum": 50\par           \}\par         \},\par         "required": [\par           "query"\par         ]\par       \}\par     \}\par   \}\par ]\par \textless{}/tools\textgreater{}\par \par For each function call, return a json object with function name and arguments within \textless{}tool\_call\textgreater{}\textless{}/tool\_call\textgreater{} XML tags:\par \textless{}tool\_call\textgreater{}\par \{"name": \textless{}function-name\textgreater{}, "arguments": \textless{}args-json-object\textgreater{}\}\par \textless{}/tool\_call\textgreater{}\par}

\medskip
\textbf{Prompt:}

{\scriptsize\ttfamily\raggedright Were Scott Derrickson and Ed Wood of the same nationality?\par}

\medskip
\textbf{Ground-Truth Answer:}

{\scriptsize\ttfamily\raggedright yes\par}
\end{tcolorbox}

\begin{tcolorbox}[example, title=Function Calling (BFCL v4) Example]
\textbf{System prompt:}

{\tiny\ttfamily\raggedright \# Tools\par \par You may call one or more functions to assist with the user query.\par \par You are provided with function signatures within \textless{}tools\textgreater{}\textless{}/tools\textgreater{} XML tags:\par \textless{}tools\textgreater{}\par Important: Always use only the latest tool list provided, ignoring any functions mentioned in previous messages.\par \{"name": "add\_to\_watchlist", "description": "This tool belongs to the trading system, which allows \par users to trade stocks, manage their account, and view stock information. Tool description: Add a \par stock to the watchlist.", "parameters": \{"type": "dict", "properties": \{"stock": \{"type": "string", \par "description": "the stock symbol to add to the watchlist. "\}\}, "required": ["stock"]\}, "response": \par \{"type": "dict", "properties": \{"watchlist": \{"type": "array", "description": "the watchlist.", \par "items": \{"type": "string"\}\}\}\}\}\par \{"name": "cancel\_order", "description": "This tool belongs to the trading system, which allows users\par  to trade stocks, manage their account, and view stock information. Tool description: Cancel an \par order.", "parameters": \{"type": "dict", "properties": \{"order\_id": \{"type": "integer", \par "description": "ID of the order to cancel. "\}\}, "required": ["order\_id"]\}, "response": \{"type": \par "dict", "properties": \{"order\_id": \{"type": "integer", "description": "ID of the cancelled order."\},\par  "status": \{"type": "string", "description": "New status of the order after cancellation \par attempt."\}\}\}\}\par \{"name": "filter\_stocks\_by\_price", "description": "This tool belongs to the trading system, which \par allows users to trade stocks, manage their account, and view stock information. Tool description: \par Filter stocks based on a price range.", "parameters": \{"type": "dict", "properties": \{"stocks": \par \{"type": "array", "items": \{"type": "string"\}, "description": "List of stock symbols to filter."\}, \par "min\_price": \{"type": "float", "description": "Minimum stock price."\}, "max\_price": \{"type": \par "float", "description": "Maximum stock price. "\}\}, "required": ["stocks", "min\_price", \par "max\_price"]\}, "response": \{"type": "dict", "properties": \{"filtered\_stocks": \{"type": "array", \par "description": "Filtered list of stock symbols within the price range.", "items": \{"type": \par "string"\}\}\}\}\}\par \{"name": "fund\_account", "description": "This tool belongs to the trading system, which allows users\par  to trade stocks, manage their account, and view stock information. Tool description: Fund the \par account with the specified amount.", "parameters": \{"type": "dict", "properties": \{"amount": \par \{"type": "float", "description": "Amount to fund the account with. "\}\}, "required": ["amount"]\}, \par "response": \{"type": "dict", "properties": \{"status": \{"type": "string", "description": "Status of \par the funding operation."\}, "new\_balance": \{"type": "float", "description": "Updated account balance \par after funding."\}\}\}\}\par \{"name": "get\_account\_info", "description": "This tool belongs to the trading system, which allows \par users to trade stocks, manage their account, and view stock information. Tool description: Get \par account information.", "parameters": \{"type": "dict", "properties": \{\}, "required": []\}, "response":\par  \{"type": "dict", "properties": \{"account\_id": \{"type": "integer", "description": "ID of the \par account."\}, "balance": \{"type": "float", "description": "Current balance of the account."\}, \par "binding\_card": \{"type": "integer", "description": "Card number associated with the account."\}\}\}\}\par \{"name": "get\_available\_stocks", "description": "This tool belongs to the trading system, which \par allows users to trade stocks, manage their account, and view stock information. Tool description: \par Get a list of stock symbols in the given sector.", "parameters": \{"type": "dict", "properties": \par \{"sector": \{"type": "string", "description": "The sector to retrieve stocks from (e.g., \par 'Technology'). "\}\}, "required": ["sector"]\}, "response": \{"type": "dict", "properties": \par \{"stock\_list": \{"type": "array", "description": "List of stock symbols in the specified sector.", \par "items": \{"type": "string"\}\}\}\}\}\par \{"name": "get\_current\_time", "description": "This tool belongs to the trading system, which allows \par users to trade stocks, manage their account, and view stock information. Tool description: Get the \par current time.", "parameters": \{"type": "dict", "properties": \{\}, "required": []\}, "response": \par \{"type": "dict", "properties": \{"current\_time": \{"type": "string", "description": "Current time in \par HH:MM AM/PM format."\}\}\}\}\par \{"name": "get\_order\_details", "description": "This tool belongs to the trading system, which allows \par users to trade stocks, manage their account, and view stock information. Tool description: Get the \par details of an order.", "parameters": \{"type": "dict", "properties": \{"order\_id": \{"type": "integer",\par  "description": "ID of the order. "\}\}, "required": ["order\_id"]\}, "response": \{"type": "dict", \par "properties": \{"id": \{"type": "integer", "description": "ID of the order."\}, "order\_type": \{"type": \par "string", "description": "Type of the order."\}, "symbol": \{"type": "string", "description": "Symbol \par of the stock in the order."\}, "price": \{"type": "float", "description": "Price at which the order \par was placed."\}, "amount": \{"type": "integer", "description": "Number of shares in the order."\}, \par "status": \{"type": "string", "description": "Current status of the order. [Enum]: [\textbackslash{}"Open\textbackslash{}", \par \textbackslash{}"Pending\textbackslash{}", \textbackslash{}"Completed\textbackslash{}", \textbackslash{}"Cancelled\textbackslash{}"]"\}\}\}\}\par \{"name": "get\_order\_history", "description": "This tool belongs to the trading system, which allows \par users to trade stocks, manage their account, and view stock information. Tool description: Get the \par stock order ID history.", "parameters": \{"type": "dict", "properties": \{\}, "required": []\}, \par "response": \{"type": "dict", "properties": \{"order\_history": \{"type": "array", "description": "List \par of orders ID in the order history.", "items": \{"type": "integer"\}\}\}\}\}\par \{"name": "get\_stock\_info", "description": "This tool belongs to the trading system, which allows \par users to trade stocks, manage their account, and view stock information. Tool description: Get the \par details of a stock.", "parameters": \{"type": "dict", "properties": \{"symbol": \{"type": "string", \par "description": "Symbol that uniquely identifies the stock. "\}\}, "required": ["symbol"]\}, "response":\par  \{"type": "dict", "properties": \{"price": \{"type": "float", "description": "Current price of the \par stock."\}, "percent\_change": \{"type": "float", "description": "Percentage change in stock price."\}, \par "volume": \{"type": "float", "description": "Trading volume of the stock."\}, "MA(5)": \{"type": \par "float", "description": "5-day Moving Average of the stock."\}, "MA(20)": \{"type": "float", \par "description": "20-day Moving Average of the stock."\}\}\}\}\par \{"name": "get\_symbol\_by\_name", "description": "This tool belongs to the trading system, which allows\par  users to trade stocks, manage their account, and view stock information. Tool description: Get the \par symbol of a stock by company name.", "parameters": \{"type": "dict", "properties": \{"name": \{"type": \par "string", "description": "Name of the company. "\}\}, "required": ["name"]\}, "response": \{"type": \par "dict", "properties": \{"symbol": \{"type": "string", "description": "Symbol of the stock or \textbackslash{}"Stock \par not found\textbackslash{}" if not available."\}\}\}\}\par \{"name": "get\_transaction\_history", "description": "This tool belongs to the trading system, which \par allows users to trade stocks, manage their account, and view stock information. Tool description: \par Get the transaction history within a specified date range.", "parameters": \{"type": "dict", \par "properties": \{"start\_date": \{"type": "string", "description": "Start date for the history (format: \par 'YYYY-MM-DD').", "default": "None"\}, "end\_date": \{"type": "string", "description": "End date for the\par  history (format: 'YYYY-MM-DD'). ", "default": "None"\}\}, "required": []\}, "response": \{"type": \par "dict", "properties": \{"transaction\_history": \{"type": "array", "description": "List of transactions\par  within the specified date range.", "items": \{"type": "dict", "properties": \{"type": \{"type": \par "string", "description": "Type of transaction. [Enum]: [\textbackslash{}"deposit\textbackslash{}", \textbackslash{}"withdrawal\textbackslash{}"]"\}, "amount": \par \{"type": "float", "description": "Amount involved in the transaction."\}, "timestamp": \{"type": \par "string", "description": "Timestamp of the transaction, formatted as 'YYYY-MM-DD HH:MM:SS'."\}\}\}\}\}\}\}\par \{"name": "get\_watchlist", "description": "This tool belongs to the trading system, which allows \par users to trade stocks, manage their account, and view stock information. Tool description: Get the \par watchlist.", "parameters": \{"type": "dict", "properties": \{\}, "required": []\}, "response": \{"type": \par "dict", "properties": \{"watchlist": \{"type": "array", "description": "List of stock symbols in the \par watchlist.", "items": \{"type": "string"\}\}\}\}\}\par \{"name": "notify\_price\_change", "description": "This tool belongs to the trading system, which \par allows users to trade stocks, manage their account, and view stock information. Tool description: \par Notify if there is a significant price change in the stocks.", "parameters": \{"type": "dict", \par "properties": \{"stocks": \{"type": "array", "items": \{"type": "string"\}, "description": "List of \par stock symbols to check."\}, "threshold": \{"type": "float", "description": "Percentage change \par threshold to trigger a notification. "\}\}, "required": ["stocks", "threshold"]\}, "response": \{"type":\par  "dict", "properties": \{"notification": \{"type": "string", "description": "Notification message \par about the price changes."\}\}\}\}\par \{"name": "place\_order", "description": "This tool belongs to the trading system, which allows users \par to trade stocks, manage their account, and view stock information. Tool description: Place an \par order.", "parameters": \{"type": "dict", "properties": \{"order\_type": \{"type": "string", \par "description": "Type of the order (Buy/Sell)."\}, "symbol": \{"type": "string", "description": "Symbol\par  of the stock to trade."\}, "price": \{"type": "float", "description": "Price at which to place the \par order."\}, "amount": \{"type": "integer", "description": "Number of shares to trade. "\}\}, "required": \par ["order\_type", "symbol", "price", "amount"]\}, "response": \{"type": "dict", "properties": \par \{"order\_id": \{"type": "integer", "description": "ID of the newly placed order."\}, "order\_type": \par \{"type": "string", "description": "Type of the order (Buy/Sell)."\}, "status": \{"type": "string", \par "description": "Initial status of the order."\}, "price": \{"type": "float", "description": "Price at \par which the order was placed."\}, "amount": \{"type": "integer", "description": "Number of shares in the\par  order."\}\}\}\}\par \{"name": "remove\_stock\_from\_watchlist", "description": "This tool belongs to the trading system, \par which allows users to trade stocks, manage their account, and view stock information. Tool \par description: Remove a stock from the watchlist.", "parameters": \{"type": "dict", "properties": \par \{"symbol": \{"type": "string", "description": "Symbol of the stock to remove. "\}\}, "required": \par ["symbol"]\}, "response": \{"type": "dict", "properties": \{"status": \{"type": "string", "description":\par  "Status of the removal operation."\}\}\}\}\par \{"name": "trading\_get\_login\_status", "description": "This tool belongs to the trading system, which \par allows users to trade stocks, manage their account, and view stock information. Tool description: \par Get the login status.", "parameters": \{"type": "dict", "properties": \{\}, "required": []\}, \par "response": \{"type": "dict", "properties": \{"status": \{"type": "boolean", "description": "Login \par status."\}\}\}\}\par \{"name": "trading\_login", "description": "This tool belongs to the trading system, which allows \par users to trade stocks, manage their account, and view stock information. Tool description: Handle \par user login.", "parameters": \{"type": "dict", "properties": \{"username": \{"type": "string", \par "description": "Username for authentication."\}, "password": \{"type": "string", "description": \par "Password for authentication. "\}\}, "required": ["username", "password"]\}, "response": \{"type": \par "dict", "properties": \{"status": \{"type": "string", "description": "Login status message."\}\}\}\}\par \{"name": "trading\_logout", "description": "This tool belongs to the trading system, which allows \par users to trade stocks, manage their account, and view stock information. Tool description: Handle \par user logout for trading system.", "parameters": \{"type": "dict", "properties": \{\}, "required": []\}, \par "response": \{"type": "dict", "properties": \{"status": \{"type": "string", "description": "Logout \par status message."\}\}\}\}\par \{"name": "withdraw\_funds", "description": "This tool belongs to the trading system, which allows \par users to trade stocks, manage their account, and view stock information. Tool description: Withdraw \par funds from the account balance.", "parameters": \{"type": "dict", "properties": \{"amount": \{"type": \par "float", "description": "Amount to withdraw from the account. "\}\}, "required": ["amount"]\}, \par "response": \{"type": "dict", "properties": \{"status": \{"type": "string", "description": "Status of \par the transaction."\}, "new\_balance": \{"type": "float", "description": "Updated account balance after \par the transaction."\}\}\}\}\par \textless{}/tools\textgreater{}\par \par For each function call, return a json object with function name and arguments within \textless{}tool\_call\textgreater{}\textless{}/tool\_call\textgreater{} XML tags:\par \textless{}tool\_call\textgreater{}\par \{"name": \textless{}function-name\textgreater{}, "arguments": \textless{}args-json-object\textgreater{}\}\par \textless{}/tool\_call\textgreater{}\par}

\medskip
\textbf{Prompt:}

{\scriptsize\ttfamily\raggedright Turn 1: Can you provide the latest trading details for Quasar Ltd.?\par Turn 2: Furthermore, I'd like to delve into the technology sector. Would you be able to compile a \par comprehensive list of stock symbols pertinent to this industry and add all of them to my watchlist\par}

\medskip
\textbf{Ground-Truth API sequence:}

{\scriptsize\ttfamily\raggedright [\par   [\par     "get\_stock\_info(symbol='QUAS')"\par   ],\par   [\par     "get\_available\_stocks('Technology')",\par     "add\_to\_watchlist(stock='AAPL')",\par     "add\_to\_watchlist(stock='GOOG')",\par     "add\_to\_watchlist(stock='MSFT')"\par   ]\par ]\par}
\end{tcolorbox}
\endgroup

\end{document}